\documentclass[preprint,12pt]{elsarticle}

\usepackage{amssymb}
\usepackage{amsmath}
\usepackage{amsthm}
\usepackage{booktabs}
\usepackage{float}
\usepackage{graphicx}
\usepackage{pdfpages}
\usepackage{changes}
\usepackage{multirow}
\usepackage[table]{xcolor}
\usepackage{algorithm}
\usepackage{algorithmic}
\usepackage{url}
\newtheorem{theorem}{Theorem}

\journal{neural network}

\begin{document}

\begin{frontmatter}



\title{TAMTRL: Teacher-Aligned Reward Reshaping for Multi-Turn Reinforcement Learning in Long-Context Compression}

\author[inst1]{Li Wang}
\author[inst1]{Yandong Wang}
\author[inst1]{Xin Yu}
\author[inst1]{Kui Zhang}
\author[inst1]{Tianhao Peng\corref{cor1}}
\author[inst1,inst2,inst3]{Wenjun Wu\corref{cor1}}
\affiliation[inst1]{organization={School of Artificial Intelligence, Beihang University},
            city={Beijing},
            postcode={100191}, 
            country={China}}
\affiliation[inst2]{organization={Hangzhou International Innovation Institute, Beihang University},
            city={Hangzhou},
            country={China}}

\affiliation[inst3]{organization={Beijing Advanced Innovation Center for Future Blockchain and Privacy Computing, Beihang University},
            city={Beijing},
            country={China}}
            
\cortext[cor1]{Corresponding author: Tianhao Peng, email: pengtianhao@buaa.edu.cn, Wenjun Wu, email: wwj09315@buaa.edu.cn}

\begin{abstract}
The rapid progress of large language models (LLMs) has led to remarkable performance gains across a wide range of tasks. However, when handling long documents that exceed the model's context window limit, the entire context cannot be processed in a single pass, making chunk-wise processing necessary. This requires multiple turns to read different chunks and update memory. However, supervision is typically provided only by the final outcome, which makes it difficult to evaluate the quality of memory updates at each turn in the multi-turn training setting. This introduces a temporal credit assignment challenge. Existing approaches, such as LLM-as-a-judge or process reward models, incur substantial computational overhead and suffer from estimation noise. To better address the credit assignment problem in multi-turn memory training, we propose Teacher-Aligned Reward Reshaping for Multi-Turn Reinforcement Learning (TAMTRL). TAMTRL leverages relevant documents as teacher signals by aligning them with each turn of model input and assigns rewards through normalized probabilities in a self-supervised manner. This provides fine-grained learning signals for each memory update and improves long-context processing. Experiments with multiple models of varying scales across seven long-context benchmarks show that TAMTRL consistently outperforms strong baselines, demonstrating its effectiveness. Our code is available at \url{https://anonymous.4open.science/r/TAMTRL-F1F8}.
\end{abstract}




\end{frontmatter}


\section{Introduction}
Large language models (LLMs) have made strong progress in reasoning~\cite{guo2025deepseek,ke2025survey}, planning~\cite{wei2025plangenllms,song2023llm}, and tool-use~\cite{jin2025search,xue2025simpletir}. Their capabilities have been further improved through reinforcement learning with verification rewards (RLVR). However, many real-world applications, such as web search~\cite{gao2025beyond} and long-document understanding~\cite{chhikara2025mem0}, require LLMs to process extremely long texts, extract relevant information, and update memory over time. Because the context window during pretraining is limited, LLMs cannot process arbitrarily long contexts. This limitation can lead to performance degradation on ultra-long documents~\cite{yu2025memagent} and remains a major challenge for practical long-text applications.

\begin{figure*}[htbp]
    \centering
    \includegraphics[width=1.0\textwidth]{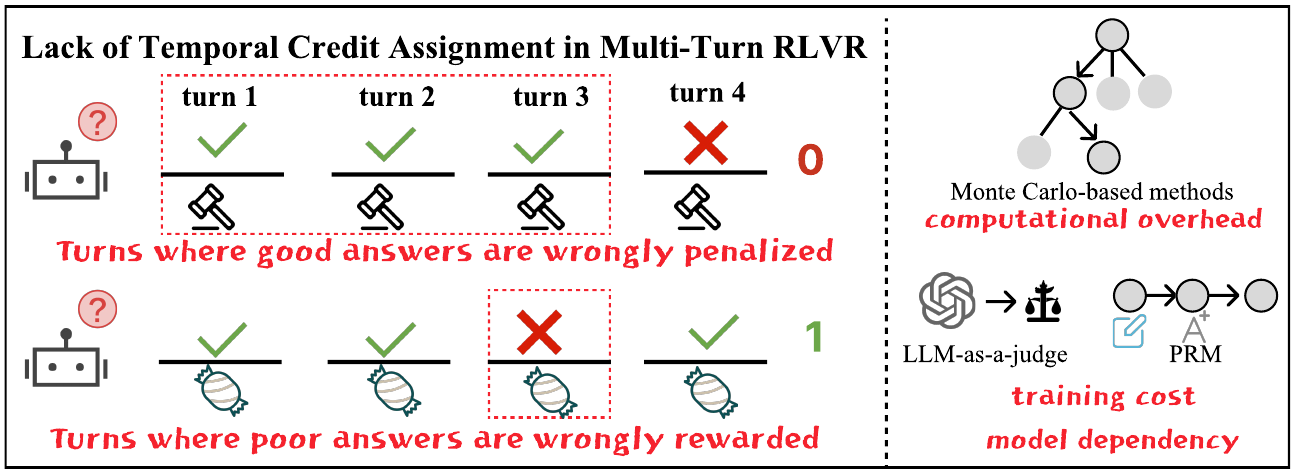}
    \caption{In multi-turn RLVR, the lack of temporal credit assignment may wrongly penalize turns with good answers or reward turns with poor answers, introducing noisy supervision that complicates training and degrades performance. Existing solutions either incur significant computational overhead or rely on external models.}
    \label{fig:motivation}
\end{figure*}

For long-context tasks, MemAgent introduces a chunk-wise processing paradigm with an explicit memory mechanism, and enhances the effective context capacity and long-document understanding of LLMs through multi-turn reinforcement learning. However, in multi-turn RLVR training, relying solely on outcome rewards poses significant challenges, as shown in Figure~\ref{fig:motivation}. Outcome-only supervision makes it difficult to evaluate intermediate memory updates. An incorrect update in one turn may cause the final answer to be wrong and penalize other valid updates. Conversely, a poor intermediate memory update may still receive reward when the final answer happens to be correct. Such uniform credit assignment lacks fine-grained process supervision and can introduce substantial noise into optimization, ultimately harming performance. To mitigate the credit assignment problem in long-horizon tasks, several approaches have been explored. Monte Carlo–based methods~\cite{kim2025astro,guan2025rstar} estimate the value of intermediate steps by sampling trajectories from each branching node, but they incur considerable computational overhead. Alternatively, approaches such as LLM-as-a-judge~\cite{stephan2025calculation,kim2024prometheus} and process reward models (PRMs)~\cite{ma2023let, zhang2025lessons} use external evaluators to assess intermediate contributions. However, they increase training cost and resource requirements. Their effectiveness also depends on the evaluator's capability, and imperfect assessments may introduce additional noise into training.

In long-document tasks, multi-turn processing makes it difficult to assign credit at each turn, which increases training difficulty and may potentially degrade performance. In this work, to address the temporal credit assignment problem in multi-turn reinforcement learning (RL) without introducing substantial additional overhead or external evaluators, we propose \textit{Teacher-Aligned Reward Reshaping for Multi-Turn Reinforcement Learning} (TAMTRL). We formalize the multi-turn long-context processing problem as a partially observable Markov decision process (POMDP). Inspired by the centralized training with decentralized execution (CTDE) paradigm in RL~\cite{lowe2017multi,yu2022surprising}, TAMTRL employs a teacher model with a more global perspective to provide centralized supervision during training. Specifically, a turn-level reward is assigned to the student model by scoring its responses using teacher-derived probabilities. The student model is then trained with multi-turn RL using these turn-level rewards, which enables more fine-grained credit assignment for evaluating memory updates at each turn and improves long-context processing and information extraction. Our method uses the model itself as the teacher, leveraging CTDE-style training to enhance long-context capabilities while avoiding reliance on external models. During credit assignment, the probabilities are obtained with only a single forward pass, without requiring relatively expensive autoregressive rollouts, which keeps the computational overhead low. By training with annotations that provide a more global perspective on the training documents, our method improves the model’s generalization capability, enabling it to identify key information from unannotated documents at test time and achieve better performance. Extensive experiments on Qwen3-0.6B and Qwen3-1.7B across seven long-context benchmarks demonstrate the effectiveness of TAMTRL. Our main contributions are summarized as follows:
\begin{itemize}
\item We formalize sequential long-document processing as a partially observable Markov decision process (POMDP) and develop a CTDE-style training framework. During training, a teacher with access to global context provides centralized supervision. At execution time, the student updates memory based only on local observations.

\item We propose \textit{Teacher-Aligned Reward Reshaping for Multi-Turn Reinforcement Learning} (TAMTRL), which employs a model with a more global perspective as a teacher to assign turn-level credit assignment to a student model operating under partial observability. By leveraging teacher-derived probabilistic scoring, TAMTRL enables fine-grained credit assignment for multi-turn RL without relying on external evaluators and incurs minimal additional overhead.

\item We decompose the optimization objective of TAMTRL and analyze its rationale from a theoretical perspective.

\item Extensive experiments across multiple models of varying scales on seven benchmarks demonstrate that TAMTRL consistently outperforms strong baselines. The results also validate its effectiveness on long-context tasks. We further conducted exploratory and analytical experiments, systematically studying the effects of different reward designs, chunk sizes, information density, and training data on the performance of TAMTRL.
\end{itemize}

\section{Related work}
\subsection{Reinforcement Learning for LLMs}
Reinforcement learning has proven effective in enhancing the capabilities of LLMs. The successes of DeepSeek-R1~\cite{guo2025deepseek} and Kimi K1.5~\cite{team2025kimi}, among others, have demonstrated that RLVR training is a valuable approach for improving model performance through environmental interactions. The introduction of algorithms such as GRPO~\cite{shao2024deepseekmath}, PPO~\cite{schulman2017proximal}, Reinforce++~\cite{hu2025reinforce++}, and DAPO~\cite{yu2025dapo} has further amplified the effectiveness of RL. In applications such as search~\cite{jin2025search} and code-related tasks~\cite{xue2025simpletir,magister2023teaching}, many studies~\cite{dong2025agentic,li2025torl} have leveraged RLVR to enhance the model’s ability to interact with tools, expanding its capabilities~\cite{lin2025understanding} and demonstrating the potential of RL.

\subsection{Long Context Compression}
Long-context handling is a crucial adaptation for LLMs in real-world applications. A variety of studies have explored enhancing long-context processing capabilities from multiple perspectives. At the model level, innovations such as Recurrent Neural Networks (RNNs)~\cite{peng2023rwkv}, linear attention mechanisms~\cite{child2019generating,katharopoulos2020transformers}, sparse attention~\cite{beltagy2020longformer,zhao2019explicit}, State Space Models (SSMs)~\cite{gu2021efficiently}, and LongRoPE~\cite{ding2024longrope} have been proposed to improve the model’s contextual capacity at the architectural level. Additionally, some research has introduced memory mechanisms~\cite{hu2025memory,chhikara2025mem0} to further enhance context handling. These include independent storage structures and memory retrieval mechanisms~\cite{zhong2023memorybank,chen2024meminsight,fang2025memp} designed for processing long contexts, as well as iterative memory updates that maintain a fixed-length memory for handling extended contexts~\cite{yu2025memagent}. Moreover, RLVR training~\cite{zhang2025memory,wang2026infmem} has been employed to improve the model’s ability to extract and update information. Furthermore, multi-agent collaboration~\cite{li2024chain} has been proposed to further refine memory management. Despite these advances, existing memory mechanisms in RL-based multi-turn training typically rely solely on outcome rewards~\cite{yu2025memagent,wang2025mem,yuan2025memsearcher}, lacking supervision over the process, which may lead to performance degradation. In contrast, TAMTRL introduces a turn-level credit assignment approach that does not rely on external models, improving training effectiveness.

\subsection{Temporal Credit Assignment in Reinforcement Learning}
The temporal credit assignment problem~\cite{pignatelli2023survey} in long-horizon interactions hampers RL efficiency, and several studies have proposed different approaches to address this challenge. Some methods heuristically assign credit based on time~\cite{sutton1988learning,sutton2016emphatic} or auxiliary goals~\cite{sutton2011horde}, while others focus on reassigning or re-evaluating credit in hindsight~\cite{andrychowicz2017hindsight,schmidhuber2019reinforcement}. Additionally, several approaches leverage natural language processing (NLP) techniques for sequence modeling~\cite{janner2021offline,chen2021decision}, where credit is assigned based on predicted probabilities. In the domain of LLMs, some methods have been developed to address the temporal credit assignment problem. Model-based approaches evaluate the contribution of intermediate steps using techniques such as critics~\cite{schulman2017proximal}, process reward models (PRM)~\cite{ma2023let, zhang2025lessons}, or LLM-judges~\cite{li2024llms,kim2024prometheus}. On the other hand, Monte Carlo methods~\cite{kim2025astro,guan2025rstar} assess the contribution of intermediate steps by performing Monte Carlo simulations at each intermediate node. Counterfactual methods~\cite{bo2024reflective} attribute the contribution to intermediate nodes to evaluate their effect. However, these methods typically involve high complexity or rely on external models, leading to increased training resource overhead. In contrast, our approach leverages the model's own local-global perspective as a teacher model, avoiding reliance on external models and maintaining relatively low overhead.

\section{Preliminary: DAPO}
In this section, we introduce Decoupled Clip and Dynamic sAmpling Policy Optimization (DAPO)~\cite{yu2025dapo}. We begin by reviewing the Group Relative Policy Optimization (GRPO)~\cite{shao2024deepseekmath} algorithm. Given a dataset \(D\) consisting of questions \(q\) and their corresponding ground-truth answers \(a\), GRPO generates rollout trajectories \(\{o_1, \dots, o_G\}\) using the previous policy \(\pi_{\theta_{\text{old}}}\) and updates the current policy \(\pi_\theta\) by maximizing the following objective:
\begin{align}
\mathcal{J}_{\text{GRPO}}(\theta) &= \mathbb{E}_{q \sim P(Q), \{o_i\} \sim \pi_{\theta}} \Bigg[ \frac{1}{G} \sum_{i=1}^G \frac{1}{|o_i|} \sum_{t=1}^{|o_i|} 
\Big( s_{i,t} - \beta \, \mathbb{D}_{\text{KL}} \left[ \pi_\theta \,\|\, \pi_{\text{ref}} \right] \Big) \Bigg], \nonumber \\
s_{i,t} &= \min \Big( \rho_{i,t} \, A_{i,t}, \, \text{clip}(\rho_{i,t}, 1 - \epsilon, 1 + \epsilon) \, A_{i,t} \Big), \nonumber \\
\rho_{i,t} &= \frac{\pi_{\theta}(o_{i,t} \mid q, o_{i,<t})}{\pi_{\theta_{\text{old}}}(o_{i,t} \mid q, o_{i,<t})},
\label{equ:grpo_loss}
\end{align}
where \(\epsilon\) and \(\beta\) are hyperparameters controlling the clipping range of the importance sampling ratio and the strength of the KL penalty, respectively. The advantage term \(A_{i,t}\) is defined as:
\begin{equation}
A_{i,t} = \frac{r_i - \mathrm{mean}(\{r_1, \dots, r_G\})}{\mathrm{std}(\{r_1, \dots, r_G\})},
\end{equation}
where \(r_i\) is the reward for trajectory \(o_i\), computed via a rule-based verification procedure. DAPO introduces several improvements over GRPO to enhance training stability and efficiency. It employs the \emph{Clip-Higher} strategy to decouple upper and lower clipping thresholds, a \emph{dynamic sampling} scheme to filter extreme trajectories and oversample informative ones, a \emph{token-level policy gradient loss} to ensure each token contributes equally, and a \emph{soft over-length penalty} to gradually penalize excessively long responses. The resulting training objective is
\begin{align}
\mathcal{J}_{\text{DAPO}}(\theta) =
\mathbb{E}_{\substack{(q,a)\sim\mathcal{D},\\\{o_i\}_{i=1}^G\sim\pi_{\theta_{\text{old}}}(\cdot|q)}}
\Bigg[
\frac{1}{\sum_{i=1}^G |o_i|}
\sum_{i=1}^G \sum_{t=1}^{|o_i|}
\min\Big(
r_{i,t}(\theta)\hat{A}_{i,t}, \nonumber \\
\mathrm{clip}\big(r_{i,t}(\theta),1-\varepsilon_{\text{low}},1+\varepsilon_{\text{high}}\big)\hat{A}_{i,t}
\Big)
\Bigg], \nonumber \\
\text{s.t.} \quad
0 < \left|\{o_i \mid \text{is\_equivalent}(a, o_i)\}\right| < G,
\label{equ:dapo}
\end{align}
where \(\varepsilon_{\text{low}}\) and \(\varepsilon_{\text{high}}\) denote the lower and upper clipping bounds, respectively, and \(\hat{A}_{i,t}\) is the advantage estimate at token \(t\) in trajectory \(o_i\).

\section{Method}
In this section, we introduce the proposed Teacher-Aligned Reward Reshaping for Multi-Turn Reinforcement Learning (TAMTRL) method, as shown in Figure \ref{fig:method}. We first present the overall workflow in Section~\ref{sebsec:workflow} and formalize the problem in Section~\ref{sebsec:formulation}. Next, in Section~\ref{sebsec:tarr}, we explain Teacher-Aligned Reward Reshaping and compute the normalized teacher probability score, $\hat{p}_{ij}$. Finally, we apply the obtained $\hat{p}_{ij}$ for multi-turn RL in Section~\ref{sebsec:multi_turn_rl}. The following sections provide a detailed description of each part.
\begin{figure*}[htbp]
    \centering
    \includegraphics[width=1.0\textwidth]{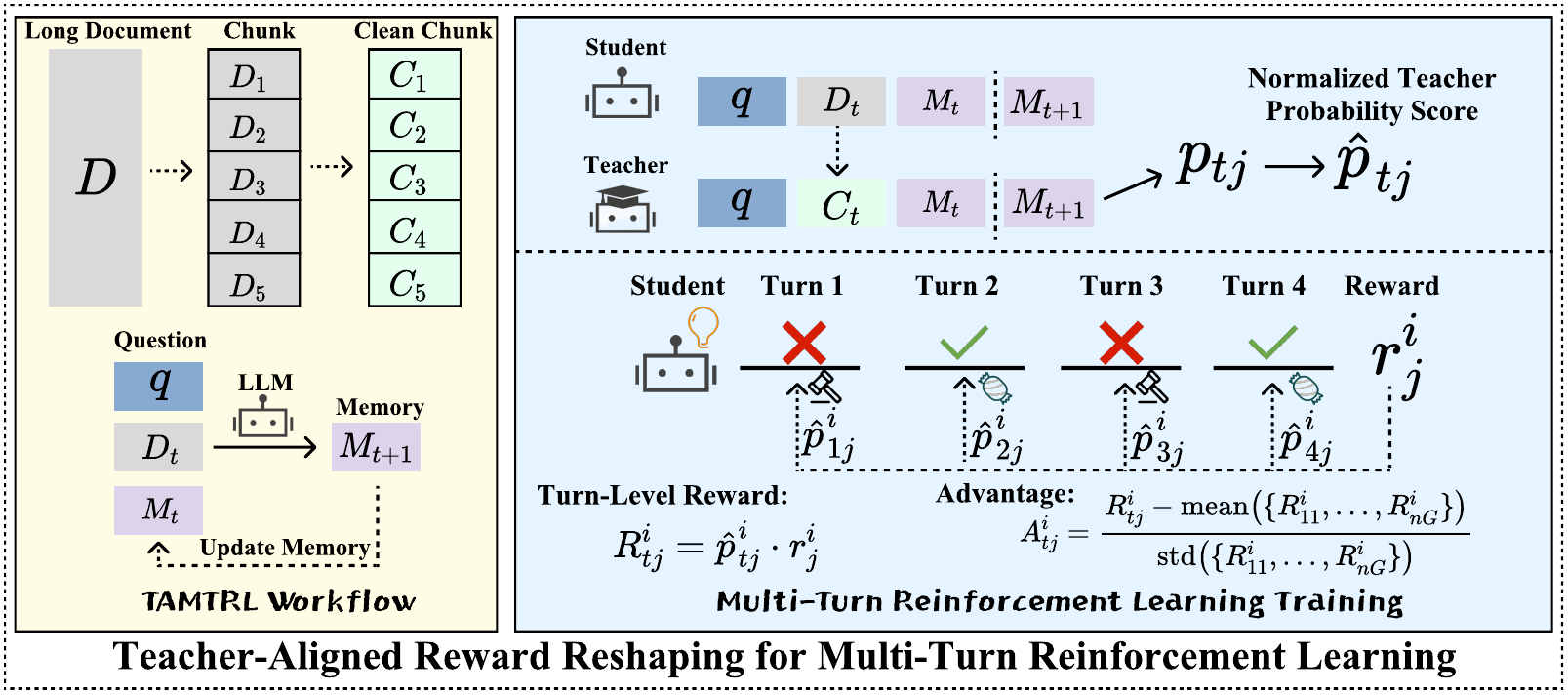}
    \caption{Framework of the TAMTRL Method. Solving long-context problems through chunk-based processing, utilizing the probability scores from a teacher with local and global views for turn-level credit assignment to enable multi-turn reinforcement learning.}
    \label{fig:method}
\end{figure*}

\subsection{TAMTRL Workflow}
\label{sebsec:workflow}
We adopt a processing pipeline analogous to that of MemAgent~\cite{yu2025memagent}. Given an input query $q$ and a long document $D$, the document is first segmented into $n$ chunks, $D_1, \dots, D_n$, according to a predefined context length $l_c$. At each time step $t$, the model maintains a natural language memory $M_t$ with a maximum length of $l_m$, initialized as an empty memory $M_0$. The model then receives $[q, D_t, M_t]$ as input (where $[.,.]$ denotes concatenation) and updates the memory to produce a new memory $M_{t+1}$ based on the information contained in the current document chunk $D_t$. This linear, stepwise processing continues until the entire document has been consumed, ultimately yielding the pair $[q, M_n]$ from which the final answer $y$ is generated. The entire process can be formulated as a POMDP, in which at each step the model only has access to the local document chunk and must decide which information to retain, ultimately producing the updated memory $M_t$, as illustrated in Section~\ref{sebsec:formulation}, thereby increasing the challenge of training.
 
\subsection{Problem Formulation}
\label{sebsec:formulation}
We formalize the sequential long-document processing problem as a Partially Observable Markov Decision Process (POMDP), defined by the tuple $\mathcal{M} = (\mathcal{S}, \mathcal{A}, \mathcal{T}, \Omega, \mathcal{O}, \mathcal{R}, \gamma)$. A document $\mathcal{D}$ is partitioned into $n$ sequential chunks $\{D_1, \dots, D_n\}$, which the LLM processes iteratively while maintaining a persistent memory $M_t$. The state at time $t$ is defined as $s_t = (q, \mathcal{D}, M_t) \in \mathcal{S}$, encapsulating the static global document and the LLM’s internal memory from the preceding step. Due to context-length constraints, the LLM cannot directly observe the full state $s_t$; instead, it receives a local observation $o_t = (q, D_t, M_t) \in \Omega$ through the observation function $\mathcal{O}(o_t \mid s_t)$. The action space $\mathcal{A}$ corresponds to the entire text space $\mathcal{V}^*$, where the action $a_t \in \mathcal{A}$ at each step represents the generated memory output $M_{t+1}$ intended for the next step. Consequently, the state transition $\mathcal{T}(s_{t+1} \mid s_t, a_t)$ is driven by the update of the memory component from $M_t$ to $M_{t+1}$. The reward function $\mathcal{R}$ is defined as a consistent signal: for all steps $t \leq n$, the reward $r_t = 1$ if the LLM correctly answers the query $q$ based on the memory $M_t$, and $0$ otherwise. The objective is to learn an optimal policy $\pi(a_t \mid o_t)$ that maximizes the expected cumulative discounted return $J(\pi) = \mathbb{E}_{\pi} [ \sum_{t=1}^n \gamma^{t-1} r_t ]$, which necessitates that the LLM strategically encodes and preserves task-relevant information within the memory across the sequential processing bottleneck.

\subsection{Teacher-Aligned Reward Reshaping}
\label{sebsec:tarr}
In conventional RLVR, supervision is provided solely through the final reward, which often fails to accurately attribute the contribution of each step in multi-turn scenarios, further exacerbating the difficulty of training. To address this, we propose Teacher-Aligned Reward Reshaping for Multi-Turn Reinforcement Learning (TAMTRL) for multi-turn credit assignment. At each turn $t$, given a document chunk $D_t$, we first remove irrelevant content based on the ground-truth document annotations, yielding a filtered chunk $C_t$ that contains only relevant information. We then concatenate the student model's input query $q$ with the memory $M_t$ and the filtered chunk $C_t$, forming the input $[q, C_t, M_t]$ for the teacher model, which is aligned with the student model. Note that the student and teacher models are the same model $\pi_{\theta}$, with different input contexts. Next, we feed the student model with $[q, D_t, M_t]$ to obtain the updated memory $M_{t+1}$. The teacher-aligned reward for this memory update is computed as the average token-wise probability assigned by the teacher model $\pi_{\theta}$:
\begin{equation}
p_t = \frac{1}{|M_{t+1}|} \sum_{i=1}^{|M_{t+1}|} \pi_{\theta}(m^{(i)}_{t+1} \mid [q, C_t, M_t]),
\end{equation}
where $m^{(i)}_{t+1}$ denotes the $i$-th token in the updated memory $M_{t+1}$ and $|\cdot|$ represents the number of tokens. However, the average probability output by the teacher model may vary significantly in magnitude, and directly using it as a reward can result in values that are too large or too small, leading to unstable training. To stabilize learning, we further normalize the scores. For all responses across the dataset, we split them by turn, obtaining $M^1_{11}, \dots, M^i_{tj}, \dots, M^s_{nG}$, where $M^i_{tj}$ denotes the $i$-th query in the $t$-th turn from the $j$-th response in the group, with a total of $s$ queries, $n$ turns, and each query rollout $n$ times. We perform min-max normalization on the turn-level teacher probabilities as follows:
\begin{equation}
\hat{p}^i_{tj} = \frac{p^i_{tj} - \min\{p^1_{11}, \dots,  p^s_{nG}\}}{\max\{p^1_{11}, \dots, p^s_{nG}\} - \min\{p^1_{11}, \dots, p^s_{nG}\}},
\end{equation}
where $p^i_{tj}$ is the average token-wise probability computed by the teacher for $M^i_{tj}$, and the minimum and maximum are taken over the entire set of responses $\{p^1_{11}, \dots,  p^s_{nG}\}$. This produces a normalized teacher probability score $\hat{p}^i_{tj}$ for each sample, maintaining relatively stable magnitudes, which is then used for reward shaping.

\subsection{Multi-Turn Reinforcement Learning Training}
\label{sebsec:multi_turn_rl}
Using the normalized teacher probability scores $\hat{p}^i_{tj}$, we conduct multi-turn RL with explicit turn-level credit assignment. For each query $q^i$, we first compute a final outcome reward $r^i_j$ for $j$-th rollout using Exact Match (EM) against the ground-truth answer $\hat{y}^i_j$, where $r^i_j = 1$ for a correct response and $r^i_j = 0$ otherwise. We then decompose each trajectory into $n$ turns, corresponding to the intermediate memory states $M^i_{1j}, \dots, M^i_{nj}$ and the final response $y^i_j$. Turn-level rewards are assigned by modulating the normalized teacher probability score with the outcome reward, resulting in
\begin{equation}
R^i_{tj} = \hat{p}^i_{tj}\cdot r^i_j,
\end{equation}
where $R^i_{tj}$ denotes the reward at the $t$-th turn of the $j$-th rollout for query $q^i$. This design ensures that the correctness of the final answer governs the overall supervision signal, thereby avoiding conflicts between optimization objectives: incorrect answers yield zero rewards across all turns, whereas correct answers receive turn-specific rewards proportional to $\hat{p}^i_{tj}$. A higher teacher probability indicates stronger alignment between the student’s memory update and the teacher’s local–global perspective, suggesting reduced interference from irrelevant content and thus yielding a larger reward; conversely, lower probabilities reflect weaker alignment and result in smaller rewards. Through this mechanism, credit is distributed across turns, providing finer-grained supervision. For each query, we sample $G$ trajectories and decompose them at the turn level into $nG$ independent samples, which together form a group. The advantages are estimated using
\begin{equation}
A^i_{tj} = \frac{R^i_{tj} - \mathrm{mean}\big(\{R^i_{11}, \dots, R^i_{nG}\}\big)}{\mathrm{std}\big(\{R^i_{11}, \dots, R^i_{nG}\}\big)},
\end{equation}
and the policy is optimized using the DAPO algorithm~\cite{yu2025dapo} according to Eq.~\ref{equ:dapo}. The algorithm pseudocode of TAMTRL is shown in~\ref{appendix:pseudocode}.

\section{Theoretical Analysis}
\label{sec:theory_analysis}
We now turn to a theoretical analysis and interpretation of the TAMTRL optimization objective. In particular, we begin by presenting the following decomposition theorem.

\begin{theorem}[Information-Theoretic Decomposition of the TAMTRL Objective]
\label{thm:tamtrl_decomp}
Consider an optimization step at state $S_t = (q, D_t, M_t)$ within a multi-step reasoning process. Let $\pi_\theta(\cdot \mid S_t)$ denote the policy generating the next memory $M_{t+1}$, and let $r_i \in \{0,1\}$ be the binary indicator of final task success, whose distribution depends on $M_{t+1}$ and the subsequent rollout. Given a teacher log-likelihood score $\hat{p}_{\text{t}} = \log \pi_{\text{teacher}}(M_{t+1} \mid S_t)$ and a reference policy $\pi_{\text{ref}}$, the TAMTRL objective is defined as:
\begin{equation}
    \mathcal{J}(\theta) \;=\; \mathbb{E}_{M_{t+1} \sim \pi_\theta}\bigl[\hat{p}_{\text{t}} \cdot r_i \bigr] \;-\; \beta \, D_{\mathrm{KL}}\bigl[\pi_\theta(\cdot\mid S_t) \;\|\; \pi_{\text{ref}}(\cdot\mid S_t)\bigr]. \nonumber
\end{equation}
This objective $\mathcal{J}(\theta)$ can be exactly decomposed into a weighted sum comprising a success-conditional optimization term, a failure-conditional regularization term, and a memory-reward mutual information term:
\begin{equation}
    \mathcal{J}(\theta) \;=\; P(r_i=1 \mid S_t) \cdot \mathcal{L}_{\text{succ}}(\theta) \;+\; P(r_i=0 \mid S_t) \cdot \mathcal{L}_{\text{fail}}(\theta) \;+\; \beta \, I_{\pi_\theta}(M_{t+1} ; r_i \mid S_t), \nonumber
\end{equation}
where the components are defined as:
\begin{align}
    \mathcal{L}_{\text{succ}}(\theta) &= \mathbb{E}_{\pi_\theta}\bigl[\log \pi_{\text{teacher}} \mid r_i=1\bigr] - \beta \, D_{\mathrm{KL}}\bigl[\pi_\theta(\cdot \mid S_t, r_i=1) \;\|\; \pi_{\text{ref}}(\cdot \mid S_t)\bigr], \nonumber \\
    \mathcal{L}_{\text{fail}}(\theta) &= - \beta \, D_{\mathrm{KL}}\bigl[\pi_\theta(\cdot \mid S_t, r_i=0) \;\|\; \pi_{\text{ref}}(\cdot \mid S_t)\bigr], \nonumber
\end{align}
and $I_{\pi_\theta}(M_{t+1} ; r_i \mid S_t)$ represents the mutual information between the generated memory $M_{t+1}$ and the outcome $r_i$ given state $S_t$.
\end{theorem}

The proof is provided in~\ref{appendix:proof_of_theorem1}. Theorem~\ref{thm:tamtrl_decomp} provides a rigorous theoretical justification for the superiority of TAMTRL over conventional outcome-only reward mechanisms in long-horizon multi-turn RL. By decomposing the optimization objective into a weighted combination of a success-conditional teacher alignment term, a failure-conditional regularization term, and a mutual information term between the memory update and the final outcome, TAMTRL achieves precise temporal credit assignment. For successful trajectories, $\mathcal{L}_{\rm succ}(\theta)$ encourages the model to closely mimic the teacher's response while adhering to the KL constraint, thereby increasing the probability of generating the teacher's response. In contrast, for failed attempts, the model is not penalized but is instead required to maintain the KL constraint through $\mathcal{L}_{\text{fail}}(\theta)$, preventing the forgetting of previous capabilities. The mutual information term $I_{\pi_\theta}(M_{t+1} ; r_i \mid S_t)$ fosters a feedback loop between the model-generated memory $M$ and the reward increment, helping the model better predict task success or failure and thus providing a more accurate supervision signal. This decomposition provides a clearer understanding of TAMTRL’s optimization objective and the theoretical reasoning behind its design.

\section{Experiment}
\subsection{Experiment Setup}
\paragraph{Training} We construct the training set based on HotpotQA~\cite{yang2018hotpotqa}, following a synthesis procedure similar to RULER~\cite{hsieh2024ruler} and MemAgent~\cite{yu2025memagent}. Specifically, we randomly sample distractor paragraphs and interleave them with relevant paragraphs so that each prompt contains 100 paragraphs in total, requiring multi-round processing. More dataset details are provided in~\ref{appendix:datasets_statics}.

\paragraph{Baselines} To evaluate the effectiveness of TAMTRL, we compare it against the following baselines: (1) CoT Distillation, including SFT~\cite{magister2023teaching} and STaR~\cite{zelikman2024star}, where multiple CoT responses are sampled per question and correct trajectories are selected for fine-tuning; (2) Knowledge Distillation: Vanilla-KD~\cite{muralidharan2024compact}, which relies on online teacher LLM inference; (3) RL-based Methods, including MemAgent~\cite{yu2025memagent} trained with DAPO~\cite{yu2025dapo} for multi-turn RL; and (4) Process Supervision, including LLM-judge\cite{kim2024prometheus} and PRM\cite{ma2023let}, which provide turn-level reward signals using an external LLM evaluator and a turn-level reward model trained on annotated reasoning turns, respectively.

\paragraph{Implementation} We adopt Qwen3~\cite{yang2025qwen3} as the backbone model, including Qwen3-0.6B and Qwen3-1.7B, and disable the thinking mode to accelerate inference. During training, we deliberately restrict the context window to 8K tokens to evaluate extrapolation ability, allocating 1024 tokens for the query, 5000 for the context chunk, 1024 for memory, and 1024 for the output, with the remaining tokens reserved for the chat template. All experiments were implemented using the Verl~\cite{sheng2024hybridflow} and MemAgent~\cite{yu2025memagent} frameworks. The experiments were carried out using Python 3.10 and PyTorch 2.6. More implementation details are provided in~\ref{appendix:implementation_details}.

\paragraph{Evaluation} Following MemAgent~\cite{yu2025memagent}, we conduct a comprehensive evaluation across several long-context benchmarks, covering both in-domain (ID) and out-of-domain (OOD) settings. For ID  evaluation, we use a HotpotQA~\cite{yang2018hotpotqa} dataset constructed with the same protocol as the training data, where each example is augmented with distractor documents. For OOD  evaluation, we use the following benchmarks: (1) RULER-QA~\cite{hsieh2024ruler}: built with a construction procedure similar to the training setup, consisting of HotpotQA examples with distractor documents; we adopt an 8K context length to increase task difficulty. (2) Multi-keys Needle-in-a-Haystack (NIAH)~\cite{kamradt2023needle}: a long-context, low information-density benchmark that requires effective key evidence extraction; we also use an 8K context to further increase difficulty. (3) LongBench-QA: we evaluate on NarrativeQA~\cite{kovcisky2018narrativeqa}, Qasper~\cite{dasigi2021dataset}, 2WikiMultihopQA~\cite{ho2020constructing}, and MuSiQue~\cite{trivedi2022musique}, which feature comparatively shorter contexts but higher information density. More dataset details are provided in~\ref{appendix:datasets_statics}. We set the sampling temperature to 1.0 and top-p to 0.7, and adopt average@4 during evaluation to improve stability. Answer correctness is evaluated using Exact Match (EM). For questions with multiple valid answers, we additionally report a sub\_em score. Under this metric, a prediction is considered fully correct only if it contains all ground-truth elements; otherwise, the score is computed as the fraction of ground-truth items correctly covered by the prediction.

\subsection{Main Result}
We report the experimental results of Qwen3-0.6B and Qwen3-1.7B under different methods in Table~\ref{tab:main-results}. The SFT approach achieves competitive performance but consistently remains inferior to RL-based methods, highlighting the importance of learning through interaction with the environment. Furthermore, MemAgent still underperforms TAMTRL, indicating that the turn-level credit assignment introduced by TAMTRL enables the model to more effectively identify salient information and update its memory, thereby enhancing long-context reasoning capabilities. Although the LLM-judge method also introduces turn-level credit assignment, its performance is slightly worse than MemAgent, possibly due to the relatively limited quality of the judge. In addition to incurring extra computational overhead, the evaluation process may introduce noise that negatively affects learning. While the PRM model achieves comparatively strong results, it still underperforms TAMTRL. Furthermore, obtaining a well-performing PRM requires dedicated data collection and extensive training, and its incorporation during the RL phase introduces additional computational overhead. Overall, TAMTRL achieves the best performance, yielding average relative improvements of 1.87\% and 2.02\% over strong baselines on the 0.6B and 1.7B backbone models, respectively. Although our method leverages annotations that provide a more global perspective during training, it generalizes effectively to unannotated documents at test time, outperforming methods that rely on external evaluators such as LLM-as-a-judge and PRM, while maintaining lower computational overhead, further validating the effectiveness of TAMTRL.

\begin{table}[t]
\label{table:main_results}
\centering
\small
\setlength{\tabcolsep}{3.5pt}
\renewcommand{\arraystretch}{1.2}
\resizebox{\linewidth}{!}{
\begin{tabular}{lcccccccc}
\toprule
\textbf{Methods} & \textbf{HotpotQA} & \textbf{RULER-QA} & \textbf{NIAH} & \textbf{2Wikimqa} & \textbf{Musique} & 
\textbf{Narrativeqa} & \textbf{Qasper} & \textbf{Average} \\
\midrule
\multicolumn{9}{l}{\textbf{\# Qwen3-0.6B}} \\
\midrule
Base~\cite{yang2025qwen3} & 0.00 & 0.65 & 11.65 & 0.63 & 0.12 & 0.00 & 0.50 & 1.94 \\
SFT~\cite{magister2023teaching} & 34.38 & 44.80 & \textbf{97.60} & 37.75 & 18.13 & 4.00 & 10.37 & 35.29 \\
STaR~\cite{zelikman2024star} & 12.11 & 22.15 & 82.70 & 24.50 & 5.12 & 1.88 & 11.00 & 22.78 \\
Vanilla-KD~\cite{muralidharan2024compact} & 27.73 & 38.15 & 89.40 & 35.75 & 18.75 & 2.62 & 11.25 & 31.95 \\
MemAgent~\cite{yu2025memagent} & 38.87 & 47.05 & 95.04 & 40.13 & \underline{23.62} & 4.13 & 12.00 & \underline{37.26} \\
LLM-judge~\cite{kim2024prometheus} & 39.06 & 47.00 & \underline{95.35} & 38.12 & 20.13 & 4.13 & 13.25 & 36.72 \\
PRM~\cite{ma2023let} & \underline{41.02} & \underline{48.75} & 85.20 & \textbf{46.63} & 20.37 & \textbf{4.75} & \textbf{13.75} & 37.21 \\
\rowcolor{yellow!20} \textbf{TAMTRL (ours)} & \textbf{42.09} & \textbf{52.25} & 95.04 & \underline{43.00} & \textbf{24.75} & \underline{4.37} & \underline{13.50} & \textbf{39.29} \\
\midrule
\multicolumn{9}{l}{\textbf{\# Qwen3-1.7B}} \\
\midrule
Base~\cite{yang2025qwen3} & 11.33 & 12.15 & 79.00 & 6.37 & 5.12 & 1.62 & 2.37 & 16.85 \\
SFT~\cite{magister2023teaching} & 40.82 & 51.30 & 96.15 & 48.50 & 24.50 & 4.00 & \textbf{13.75} & 39.86 \\
STaR~\cite{zelikman2024star} & 25.20 & 35.65 & 70.10 & 37.13 & 16.12 & 2.75 & 11.50 & 28.35 \\
Vanilla-KD~\cite{muralidharan2024compact} & 29.49 & 44.40 & 80.15 & 41.50 & 21.00 & 3.75 & 9.25 & 32.79 \\
MemAgent~\cite{yu2025memagent} & \underline{42.97} & 51.60 & 95.30 & 43.75 & \textbf{31.37} & 5.00 & 11.63 & 40.23 \\
LLM-judge~\cite{kim2024prometheus} & 41.21 & 52.05 & 93.30 & 44.87 & 27.13 & \underline{5.25} & 12.38 & 39.46 \\
PRM~\cite{ma2023let} & 42.19 & \underline{54.90} & \underline{97.20} & \underline{49.63} & 26.88 & 4.75 & 12.75 & \underline{41.19} \\
\rowcolor{yellow!20} \textbf{TAMTRL (ours)} & \textbf{47.66} & \textbf{55.45} & \textbf{98.25} & \textbf{53.25} & \underline{29.00} & \textbf{5.50} & \underline{13.38} & \textbf{43.21} \\
\bottomrule
\end{tabular}
}
\caption{Performance (\%) of Qwen3-0.6B and Qwen3-1.7B models across seven representative benchmarks under various methods. The \textbf{bold} and \underline{underline} indicate the best and second-best results, respectively.}
\label{tab:main-results}
\end{table}

\subsection{Training Dynamics Analysis}
We present the training dynamics of TAMTRL on both the 0.6B and 1.7B models in Figures~\ref{fig:training_dynamics_06b} and~\ref{fig:training_dynamics_17b}. As shown in the figures, TAMTRL exhibits consistently stable optimization behavior throughout training. Benefiting from more fine-grained turn-level teacher supervision, the model is able to more effectively identify and retain task-relevant information, thereby improving learning efficiency and ultimately achieving superior performance. Although the LLM-judge approach also introduces turn-level credit assignment, its performance remains close to that of MemAgent and still lags behind TAMTRL. This is likely because the relatively limited capability of the judge introduces additional noise during the evaluation process, which weakens the effectiveness of the credit assignment signal. The PRM approach yields competitive performance, yet it still falls short of TAMTRL overall. In addition, training PRM entails non-trivial overhead, and the reliance on an external model further increases the demand for computational resources. In contrast, TAMTRL does not rely on any external models and achieves the best performance with comparatively lower overhead.

\begin{figure*}[t]
    \centering
    \setlength{\tabcolsep}{2pt}
    \begin{tabular}{ccc}
        \includegraphics[width=0.32\textwidth]{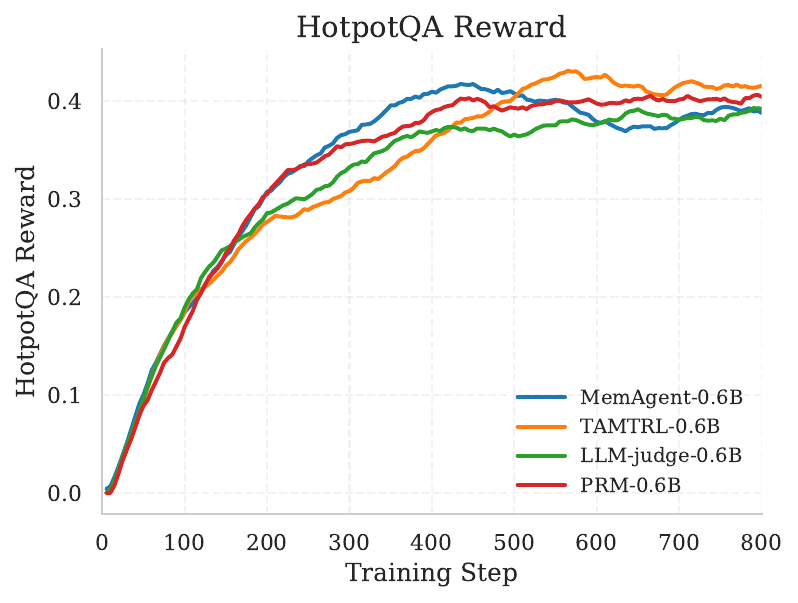} &
        \includegraphics[width=0.32\textwidth]{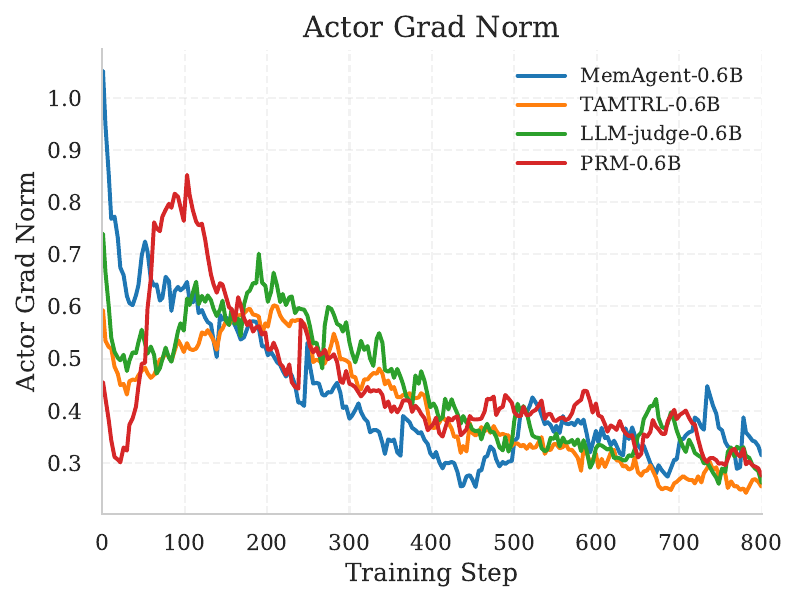} &
        \includegraphics[width=0.32\textwidth]{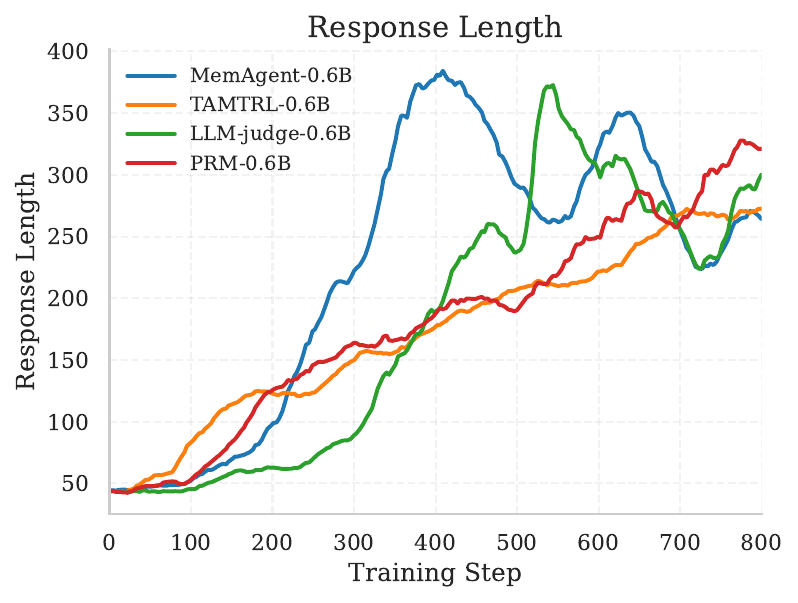} 
    \end{tabular}
    \caption{
    Training dynamics comparison of Qwen3-0.6B model.}
    \label{fig:training_dynamics_06b}
\end{figure*}

\begin{figure*}[t]
    \centering
    \setlength{\tabcolsep}{2pt}
    \begin{tabular}{ccc}
        \includegraphics[width=0.32\textwidth]{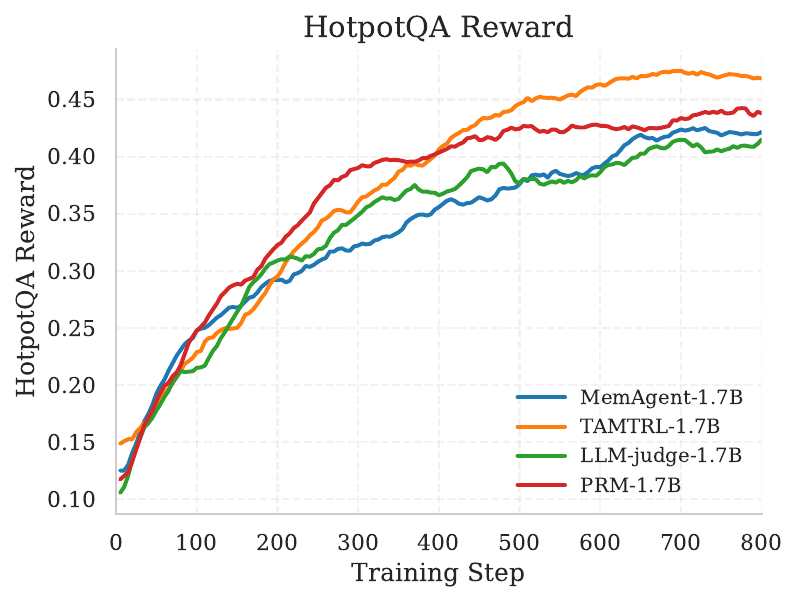} &
        \includegraphics[width=0.32\textwidth]{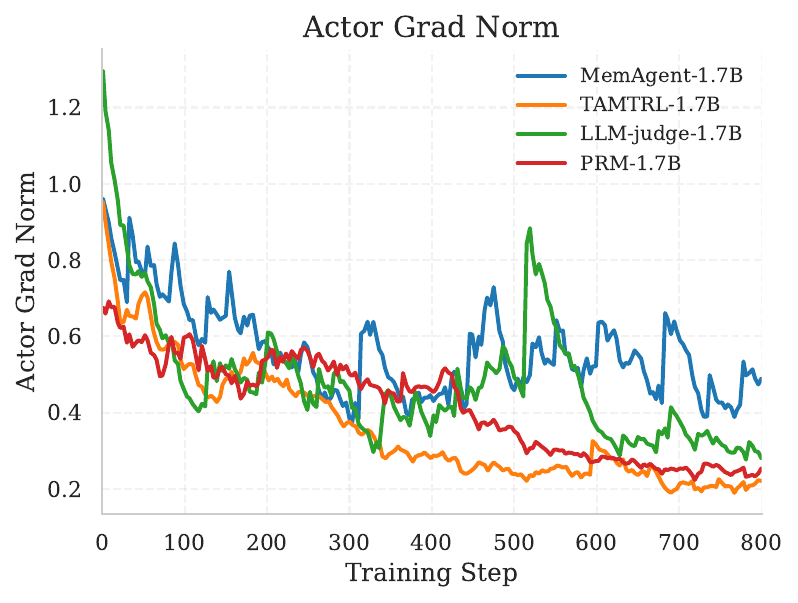} &
        \includegraphics[width=0.32\textwidth]{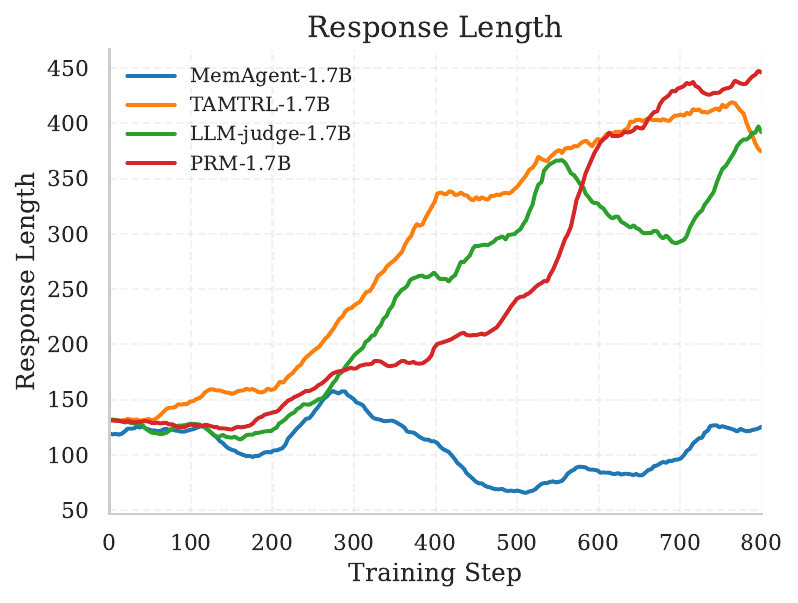} 
    \end{tabular}
    \caption{
    Training dynamics comparison of Qwen3-1.7B model.}
    \label{fig:training_dynamics_17b}
\end{figure*}

\subsection{Ablation experiment}
To evaluate the effectiveness of each module in TAMTRL, we conducted an ablation study, focusing on the role of the length normalization factor $|M_{t+1}|$ in Equation (4) and the max-min normalization in Equation (5). Specifically, we removed the length normalization factor $|M_{t+1}|$ in Equation (4) (denoted as w/o l-norm) and the max-min normalization in Equation (5) (denoted as w/o mm-norm). The experimental results are presented in Table \ref{tab:ablation-study}. The w/o l-norm configuration showed a certain degree of performance decline, likely due to the absence of length normalization. Without this factor, longer sentences tend to accumulate higher probabilities after min-max normalization, making them more likely to receive larger rewards. This may lead to the model developing a preference for longer outputs, thereby introducing optimization noise and negatively impacting performance. On the other hand, w/o mm-norm, which lacked normalization, exhibited smaller probability magnitudes and weaker rewards. After advantage normalization, this likely amplified the noise in the teacher supervision signal, leading to training failure and degraded performance. These results confirm the necessity of each module in TAMTRL.

\begin{table}[t]
\centering
\small
\setlength{\tabcolsep}{3.5pt}
\renewcommand{\arraystretch}{1.2}
\resizebox{\linewidth}{!}{
\begin{tabular}{lcccccccc}
\toprule
\textbf{Methods} & \textbf{HotpotQA} & \textbf{RULER-QA} & \textbf{NIAH} & \textbf{2Wikimqa} & \textbf{Musique} & 
\textbf{Narrativeqa} & \textbf{Qasper} & \textbf{Average} \\
\midrule
\multicolumn{9}{l}{\textbf{\# Qwen3-0.6B}} \\
\midrule
TAMTRL (ours) & \textbf{42.09} & \textbf{52.25} & \underline{95.04} & \textbf{43.00} & \textbf{24.75} & \textbf{4.37} & \textbf{13.50} & \textbf{39.29} \\
w/o l-norm & \underline{39.26} & \underline{49.70} & \textbf{95.15} & \underline{39.50} & \underline{22.25} & \underline{4.00} & \underline{13.12} & \underline{37.57} \\
w/o mm-norm & 0.00 & 0.00 & 0.00 & 0.00 & 0.00 & 0.00 & 0.00 & 0.00 \\
\midrule
\multicolumn{9}{l}{\textbf{\# Qwen3-1.7B}} \\
\midrule
TAMTRL (ours) & \textbf{47.66} & \textbf{55.45} & \textbf{98.25} & \textbf{53.25} & \textbf{29.00} & \textbf{5.50} & \textbf{13.38} & \textbf{43.21} \\
w/o l-norm & \underline{40.04} & \underline{52.55} & \underline{83.45} & \underline{44.87} & \textbf{29.00} & \underline{5.25} & \textbf{13.38} & \underline{38.36} \\
w/o mm-norm & 0.00 & 0.00 & 0.05 & 0.00 & \underline{0.00} & 0.00 & \underline{0.00} & 0.01 \\
\bottomrule
\end{tabular}
}
\caption{Ablation study on the performance (\%) of Qwen3 model across seven representative benchmarks under various methods. The \textbf{bold} and \underline{underline} indicate the best and second-best results, respectively.}
\label{tab:ablation-study}
\end{table}

\subsection{Further Exploration of Reward Designs}
\label{sebsec:reward_design_explore}
We further analyzed the impact of different reward design. Specifically, we explored consolidating all relevant documents into a single window, instead of aligning the context with the student model window-by-window. We then applied turn-level credit assignment using this approach, referred to as global-reward. Additionally, we modified Equation (6) to $R^i_{tj} = \hat{p}^i_{tj} + r^i_{j}$, which resembles the commonly used outcome reward and format reward, denoted as plus-reward. The experimental results are presented in Table \ref{tab:reward-design}. Under the global-reward setting, the model's performance showed a slight decline compared to TAMTRL. This may be due to the misalignment between the teacher's and student's contexts during each turn, meaning that some relevant documents may appear in the teacher's context but not in the student's, thus introducing noise into the credit assignment process and impairing performance. In the plus-reward setup, the performance dropped more significantly, likely because the addition of rewards effectively introduced a dual-objective optimization problem. Furthermore, since the two rewards have different magnitudes, this increased the variance in gradients, thereby complicating the model's optimization and ultimately resulting in poorer performance. These results highlight the effectiveness of the credit assignment design in TAMTRL.

\begin{table}[t]
\centering
\small
\setlength{\tabcolsep}{3.5pt}
\renewcommand{\arraystretch}{1.2}
\resizebox{\linewidth}{!}{
\begin{tabular}{lcccccccc}
\toprule
\textbf{Methods} & \textbf{HotpotQA} & \textbf{RULER-QA} & \textbf{NIAH} & \textbf{2Wikimqa} & \textbf{Musique} & 
\textbf{Narrativeqa} & \textbf{Qasper} & \textbf{Average} \\
\midrule
\multicolumn{9}{l}{\textbf{\# Qwen3-0.6B}} \\
\midrule
TAMTRL (ours) & \textbf{42.09} & \textbf{52.25} & \textbf{95.04} & \textbf{43.00} & \textbf{24.75} & \textbf{4.37} & \textbf{13.50} & \textbf{39.29} \\
global-reward & \underline{42.77} & \underline{51.50} & \underline{89.75} & \underline{41.62} & \underline{22.88} & \underline{3.62} & \underline{14.12} & \underline{38.04} \\
plus-reward & 16.02 & 19.30 & 65.00 & 21.25 & 2.50 & 1.37 & 8.75 & 19.17 \\
\midrule
\multicolumn{9}{l}{\textbf{\# Qwen3-1.7B}} \\
\midrule
TAMTRL (ours) & \textbf{47.66} & \textbf{55.45} & \textbf{98.25} & \textbf{53.25} & \textbf{29.00} & \textbf{5.50} & \textbf{13.38} & \textbf{43.21} \\
global-reward & \underline{46.29} & 51.05 & \underline{93.40} & \underline{51.50} & \underline{28.87} & 4.00 & \underline{14.00} & \underline{41.30} \\
plus-reward & 44.53 & \underline{52.95} & 80.85 & 47.87 & 21.38 & \underline{4.50} & 13.63 & 37.96 \\
\bottomrule
\end{tabular}
}
\caption{Exploration of reward design for Qwen3 model across seven representative benchmarks. The \textbf{bold} and \underline{underline} indicate the best and second-best results, respectively.}
\label{tab:reward-design}
\end{table}

\subsection{Information Density Analysis Experiment}
\subsubsection{Training}
To investigate the impact of varying training document quantities on model performance, we tested the TAMTRL model with different document quantities on the HotpotQA dataset, which features long-context passages, using a fixed test document quantity of 100. The experimental results are presented in Figure \ref{fig:dn_analysis}(left). Overall, as the number of training documents increases, the model must extract key information from a greater number of interfering documents and undergo more processing turns, thereby increasing the training complexity. Consequently, model performance experiences a moderate decline. However, the model still maintains relatively high performance, demonstrating the robustness of TAMTRL to some degree.

\subsubsection{Testing}
To investigate the impact of varying training document quantities on model performance, we evaluated the TAMTRL model on the HotpotQA dataset with different document quantities, while keeping the training document quantity fixed at 100. The experimental results are presented in Figure \ref{fig:dn_analysis}(right). As the number of documents increases, the model is required to extract key information from a greater number of potentially irrelevant documents to answer the questions. With more interaction rounds, the task difficulty increases, leading to a general decline in model performance. Despite the added complexity, the model still manages to retain a relatively good level of accuracy.

\begin{figure*}[t]
    \centering
    \setlength{\tabcolsep}{2pt}
    \begin{tabular}{ccc}
        \includegraphics[width=0.4\textwidth]{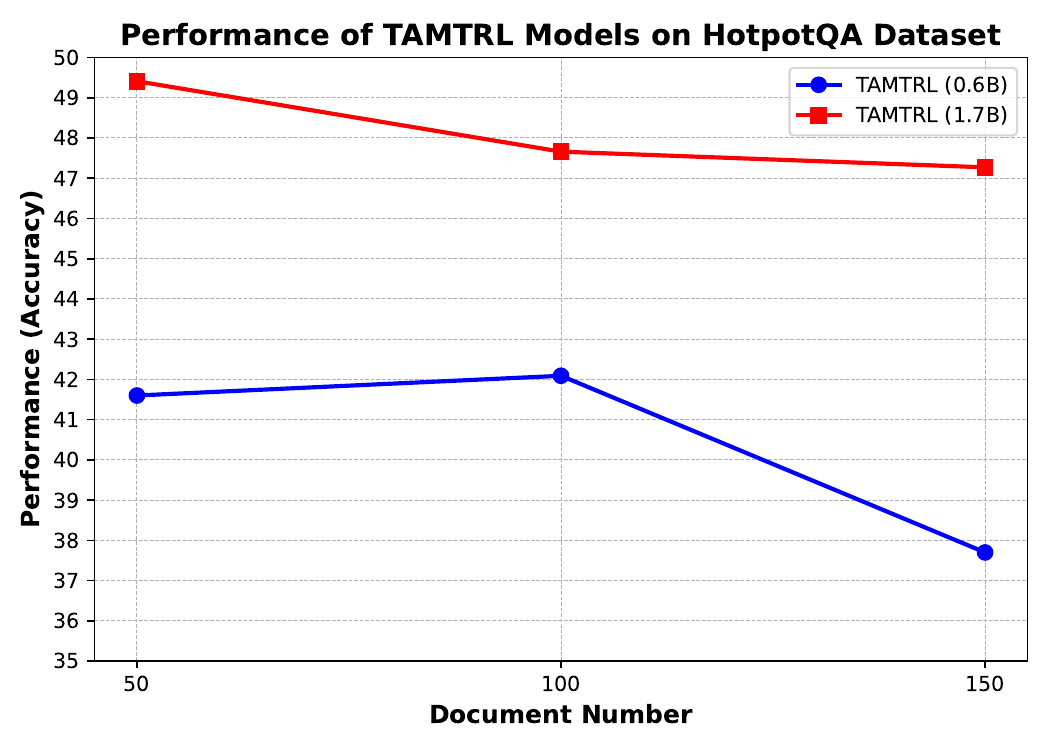} &
        \includegraphics[width=0.4\textwidth]{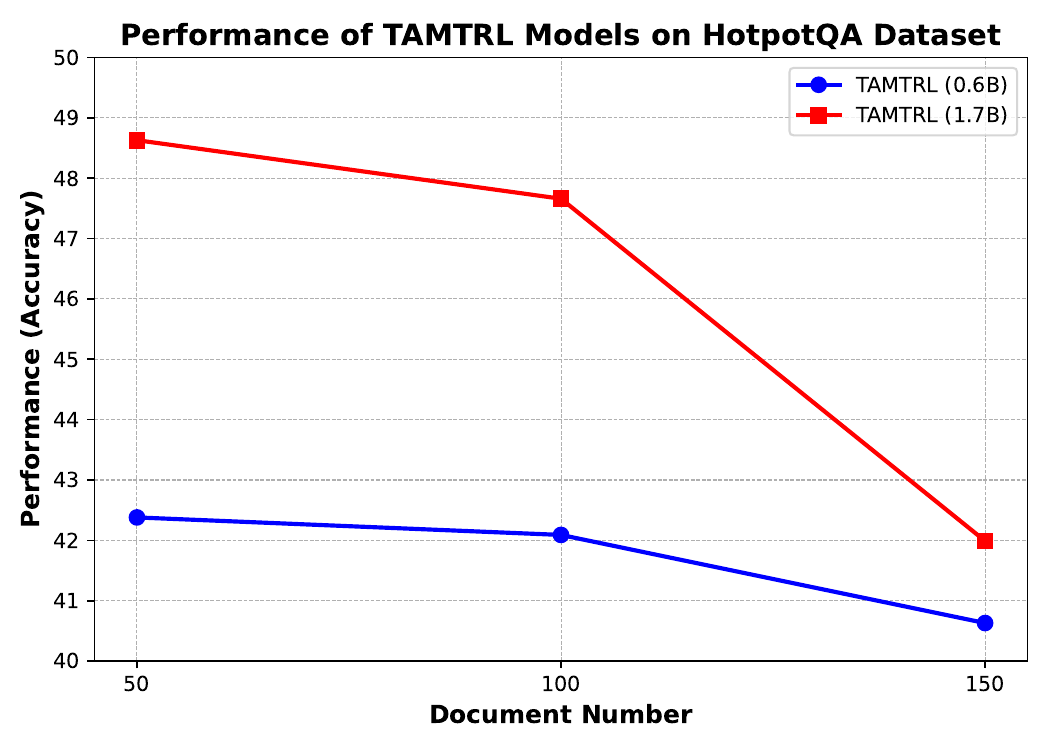}
    \end{tabular}
    \caption{Analysis of the impact of different document quantities on the performance of TAMTRL on the HotpotQA dataset. (Left) Test results with varying training document quantities at a fixed test document quantity of 100; (Right) Test results with varying test document quantities at a fixed training document quantity of 100.}
    \label{fig:dn_analysis}
\end{figure*}

\subsection{Chunk Size Analysis Experiment}
\subsubsection{Training}
To investigate the impact of varying training chunk sizes on model performance, we trained the TAMTRL model on the HotpotQA dataset using different chunk sizes, with a fixed test chunk size of 5000. The experimental results are shown in Figure \ref{fig:cs_analysis}(left). As the chunk size increases, the model is required to process more information in a turn, which increases the processing burden per turn. However, this also reduces the number of turns, mitigating risks such as noise propagation, making it a trade-off between the two factors. For the larger Qwen3-1.7B model, good performance was achieved across various chunk sizes. In contrast, for the smaller Qwen3-0.6B model, the performance showed a moderate decline as the chunk size increased, likely due to its limited ability to handle long-context sequences. This also validates the necessity of multi-turn processing, as single-pass processing of long texts may lead to a performance decline.

\subsubsection{Testing}
To investigate the impact of different test chunk sizes on model performance, we tested the TAMTRL model trained with a fixed chunk size of 5000 on the HotpotQA dataset using different test chunk sizes. The experimental results are shown in Figure \ref{fig:cs_analysis}(right). As the chunk size increases, the number of processing turns required by the model decreases, which may reduce the risk of information loss or noise propagation across multiple turns. However, at the same time, the model needs to process more of the document in a single turn, potentially increasing the cognitive load per turn. Therefore, there is a trade-off, where both excessively large and small chunk sizes may degrade performance. The performance is optimal when the chunk size is appropriately moderate.

\begin{figure*}[t]
    \centering
    \setlength{\tabcolsep}{2pt}
    \begin{tabular}{ccc}
        \includegraphics[width=0.4\textwidth]{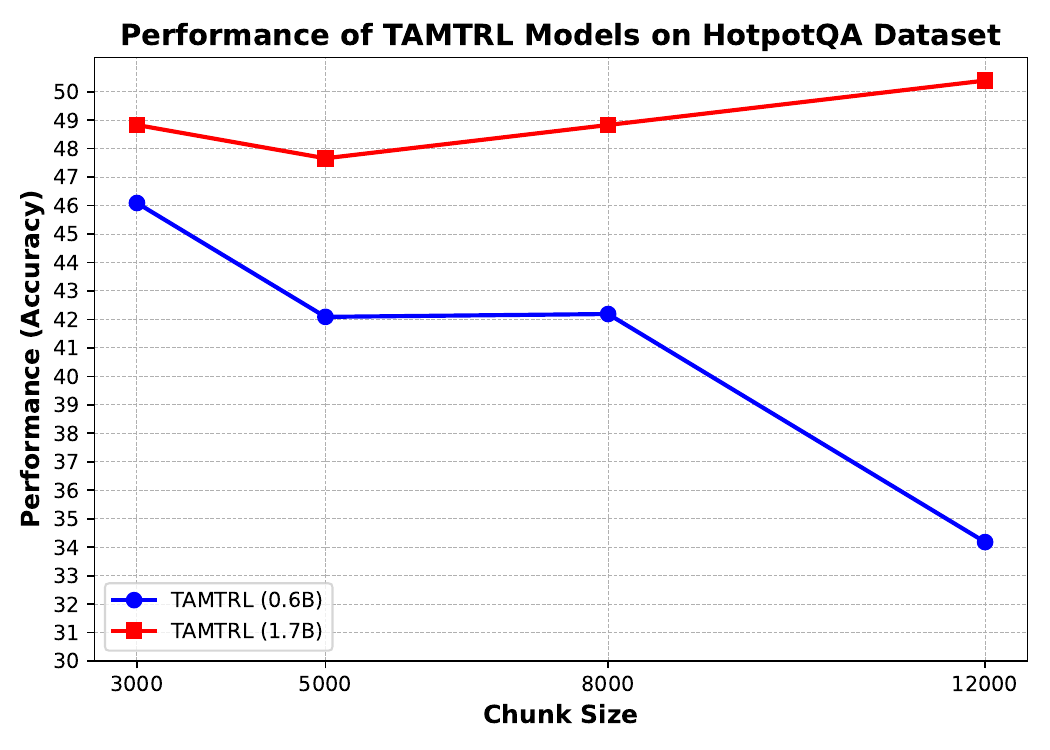} &
        \includegraphics[width=0.4\textwidth]{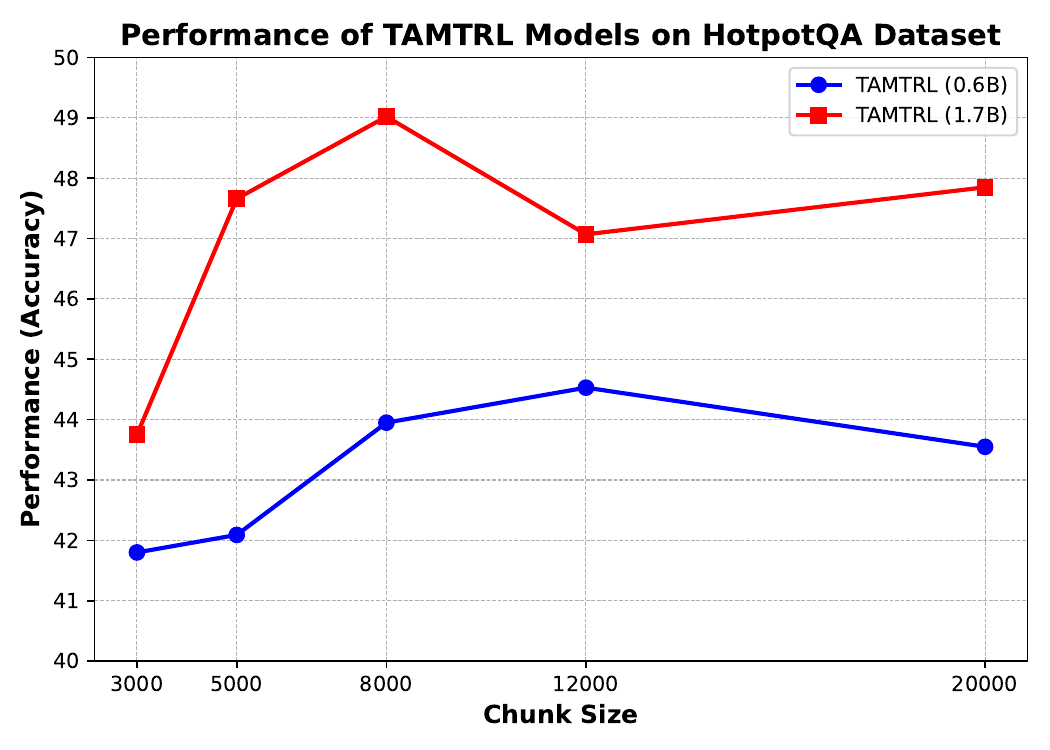}
    \end{tabular}
    \caption{Analysis of the impact of different chunk sizes on the performance of TAMTRL on the HotpotQA dataset. (Left) Test results with varying training chunk sizes at a fixed test chunk size of 5000; (Right) Test results with varying test chunk sizes at a fixed training chunk size of 5000.}
    \label{fig:cs_analysis}
\end{figure*}

\subsection{Impact of Varying Training Data Sizes}
To assess the effect of training data scale on TAMTRL, we train the model with varying proportions of the training set, as shown in Fig.~\ref{fig:data_ratio_exp}. With limited data, performance is relatively weak. As the data ratio increases, performance improves consistently and reaches a strong level at 50\% of the data. Further scaling the dataset continues to bring improvements, albeit at a slower pace, with the best performance obtained when the full dataset (100\%) is used.

\begin{figure*}[t]
    \centering
    \setlength{\tabcolsep}{2pt}
    \begin{tabular}{ccc}
        \includegraphics[width=0.4\textwidth]{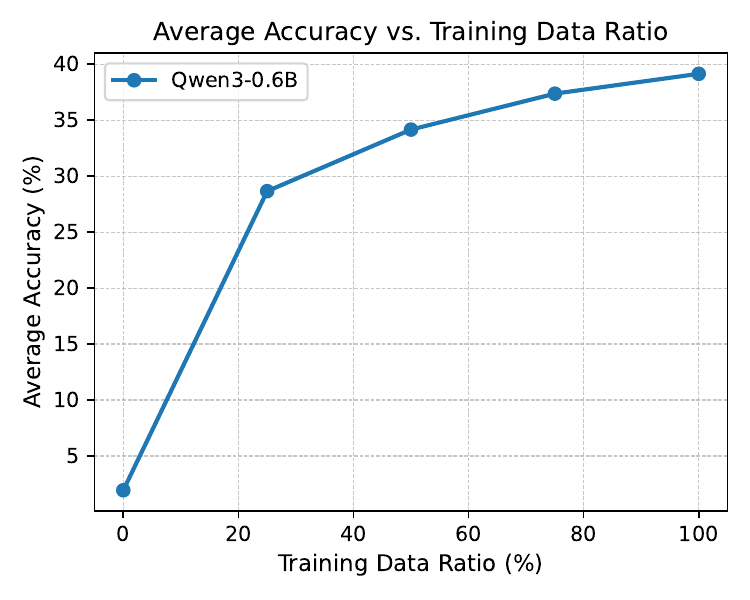} &
        \includegraphics[width=0.4\textwidth]{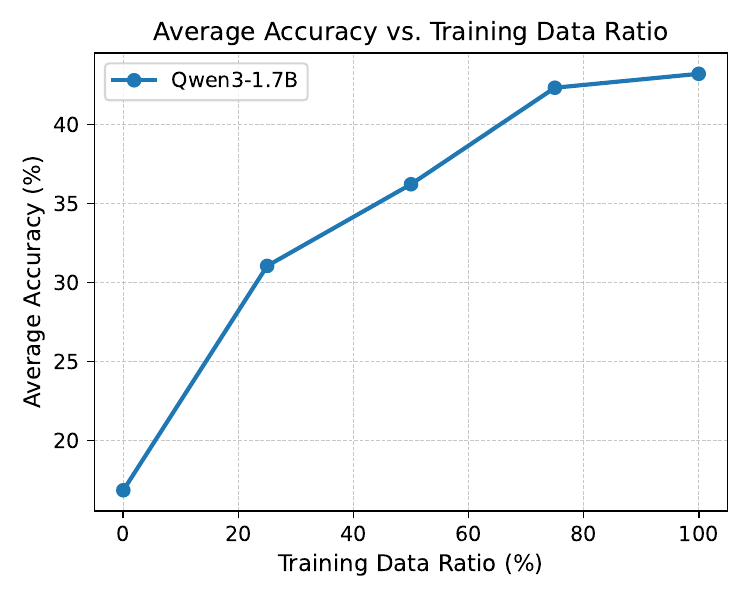}
    \end{tabular}
    \caption{Average performance variation of the model with TAMTRL under different training data ratios on seven datasets.}
    \label{fig:data_ratio_exp}
\end{figure*}

\section{Conclusion}
Long-text processing poses significant challenges for LLMs. To address this, MemAgent improves long-context modeling through chunk-wise processing, a memory mechanism, and multi-turn RL, but it still suffers from temporal credit assignment. To mitigate this issue, we first formulate sequential long-document processing as a POMDP, and then develop a CTDE-inspired teacher-student framework. Building on this formulation, we propose TAMTRL, where a teacher with privileged local-global context provides turn-level credit assignment during training to supervise a student model that performs memory updates using only local observations at execution time. This approach requires no external models and incurs minimal computational overhead. We further conduct a theoretical analysis by decomposing the optimization objective of TAMTRL to validate the rationale behind its learning goal. Experimental results on seven benchmarks using Qwen3-0.6B and Qwen3-1.7B demonstrate that TAMTRL consistently outperforms strong baseline methods, validating the effectiveness of our approach. Additionally, we explore the impact of different reward designs, varying information density, chunk sizes, and training data ratios on TAMTRL's performance through further experimental analysis.

\section*{Funding information}
This work was supported by the National Key Research and Development Program of China (Grant No. 2025YFF1505704), the National Science and Technology Major Project (Grant No. 2022ZD0117402), the National Natural Science Foundation of China (Grant No. 62441617), and the Beijing Advanced Innovation Center for Future Blockchain and Privacy Computing.

\section*{Declaration of competing interest}
The authors declare that they have no known competing financial interests or personal relationships that could have appeared to influence the work reported in this paper.

\bibliographystyle{elsarticle-num} 
\bibliography{references}

\begin{thebibliography}{10}
\expandafter\ifx\csname url\endcsname\relax
  \def\url#1{\texttt{#1}}\fi
\expandafter\ifx\csname urlprefix\endcsname\relax\def\urlprefix{URL }\fi
\expandafter\ifx\csname href\endcsname\relax
  \def\href#1#2{#2} \def\path#1{#1}\fi

\bibitem{guo2025deepseek}
D.~Guo, D.~Yang, H.~Zhang, J.~Song, P.~Wang, Q.~Zhu, R.~Xu, R.~Zhang, S.~Ma, X.~Bi, et~al., Deepseek-r1: Incentivizing reasoning capability in llms via reinforcement learning, arXiv preprint arXiv:2501.12948 (2025).

\bibitem{ke2025survey}
Z.~Ke, F.~Jiao, Y.~Ming, X.-P. Nguyen, A.~Xu, D.~X. Long, M.~Li, C.~Qin, P.~Wang, S.~Savarese, et~al., A survey of frontiers in llm reasoning: Inference scaling, learning to reason, and agentic systems, arXiv preprint arXiv:2504.09037 (2025).

\bibitem{wei2025plangenllms}
H.~Wei, Z.~Zhang, S.~He, T.~Xia, S.~Pan, F.~Liu, Plangenllms: A modern survey of llm planning capabilities, in: Proceedings of the 63rd Annual Meeting of the Association for Computational Linguistics (Volume 1: Long Papers), 2025, pp. 19497--19521.

\bibitem{song2023llm}
C.~H. Song, J.~Wu, C.~Washington, B.~M. Sadler, W.-L. Chao, Y.~Su, Llm-planner: Few-shot grounded planning for embodied agents with large language models, in: Proceedings of the IEEE/CVF international conference on computer vision, 2023, pp. 2998--3009.

\bibitem{jin2025search}
B.~Jin, H.~Zeng, Z.~Yue, J.~Yoon, S.~Arik, D.~Wang, H.~Zamani, J.~Han, Search-r1: Training llms to reason and leverage search engines with reinforcement learning, arXiv preprint arXiv:2503.09516 (2025).

\bibitem{xue2025simpletir}
Z.~Xue, L.~Zheng, Q.~Liu, Y.~Li, X.~Zheng, Z.~Ma, B.~An, Simpletir: End-to-end reinforcement learning for multi-turn tool-integrated reasoning, arXiv preprint arXiv:2509.02479 (2025).

\bibitem{gao2025beyond}
J.~Gao, W.~Fu, M.~Xie, S.~Xu, C.~He, Z.~Mei, B.~Zhu, Y.~Wu, Beyond ten turns: Unlocking long-horizon agentic search with large-scale asynchronous rl, arXiv preprint arXiv:2508.07976 (2025).

\bibitem{chhikara2025mem0}
P.~Chhikara, D.~Khant, S.~Aryan, T.~Singh, D.~Yadav, Mem0: Building production-ready ai agents with scalable long-term memory, arXiv preprint arXiv:2504.19413 (2025).

\bibitem{yu2025memagent}
H.~Yu, T.~Chen, J.~Feng, J.~Chen, W.~Dai, Q.~Yu, Y.-Q. Zhang, W.-Y. Ma, J.~Liu, M.~Wang, et~al., Memagent: Reshaping long-context llm with multi-conv rl-based memory agent, arXiv preprint arXiv:2507.02259 (2025).

\bibitem{kim2025astro}
J.~Kim, A.~Goyal, L.~Tan, H.~Hajishirzi, S.~Iyer, T.~Wang, Astro: Teaching language models to reason by reflecting and backtracking in-context, arXiv preprint arXiv:2507.00417 (2025).

\bibitem{guan2025rstar}
X.~Guan, L.~L. Zhang, Y.~Liu, N.~Shang, Y.~Sun, Y.~Zhu, F.~Yang, M.~Yang, Rstar-math: Small llms can master math reasoning with self-evolved deep thinking, arXiv preprint arXiv:2501.04519 (2025).

\bibitem{stephan2025calculation}
A.~Stephan, D.~Zhu, M.~A{\ss}enmacher, X.~Shen, B.~Roth, From calculation to adjudication: Examining llm judges on mathematical reasoning tasks, in: Proceedings of the Fourth Workshop on Generation, Evaluation and Metrics (GEM$^2$), 2025, pp. 759--773.

\bibitem{kim2024prometheus}
S.~Kim, J.~Suk, S.~Longpre, B.~Y. Lin, J.~Shin, S.~Welleck, G.~Neubig, M.~Lee, K.~Lee, M.~Seo, Prometheus 2: An open source language model specialized in evaluating other language models, in: Proceedings of the 2024 Conference on Empirical Methods in Natural Language Processing, 2024, pp. 4334--4353.

\bibitem{ma2023let}
Q.~Ma, H.~Zhou, T.~Liu, J.~Yuan, P.~Liu, Y.~You, H.~Yang, Let's reward step by step: Step-level reward model as the navigators for reasoning, arXiv preprint arXiv:2310.10080 (2023).

\bibitem{zhang2025lessons}
Z.~Zhang, C.~Zheng, Y.~Wu, B.~Zhang, R.~Lin, B.~Yu, D.~Liu, J.~Zhou, J.~Lin, The lessons of developing process reward models in mathematical reasoning, arXiv preprint arXiv:2501.07301 (2025).

\bibitem{lowe2017multi}
R.~Lowe, Y.~I. Wu, A.~Tamar, J.~Harb, O.~Pieter~Abbeel, I.~Mordatch, Multi-agent actor-critic for mixed cooperative-competitive environments, Advances in neural information processing systems 30 (2017).

\bibitem{yu2022surprising}
C.~Yu, A.~Velu, E.~Vinitsky, J.~Gao, Y.~Wang, A.~Bayen, Y.~Wu, The surprising effectiveness of ppo in cooperative multi-agent games, Advances in neural information processing systems 35 (2022) 24611--24624.

\bibitem{team2025kimi}
K.~Team, A.~Du, B.~Gao, B.~Xing, C.~Jiang, C.~Chen, C.~Li, C.~Xiao, C.~Du, C.~Liao, et~al., Kimi k1. 5: Scaling reinforcement learning with llms, arXiv preprint arXiv:2501.12599 (2025).

\bibitem{shao2024deepseekmath}
Z.~Shao, P.~Wang, Q.~Zhu, R.~Xu, J.~Song, X.~Bi, H.~Zhang, M.~Zhang, Y.~Li, Y.~Wu, et~al., Deepseekmath: Pushing the limits of mathematical reasoning in open language models, arXiv preprint arXiv:2402.03300 (2024).

\bibitem{schulman2017proximal}
J.~Schulman, F.~Wolski, P.~Dhariwal, A.~Radford, O.~Klimov, Proximal policy optimization algorithms, arXiv preprint arXiv:1707.06347 (2017).

\bibitem{hu2025reinforce++}
J.~Hu, Reinforce++: A simple and efficient approach for aligning large language models, arXiv e-prints (2025) arXiv--2501.

\bibitem{yu2025dapo}
Q.~Yu, Z.~Zhang, R.~Zhu, Y.~Yuan, X.~Zuo, Y.~Yue, W.~Dai, T.~Fan, G.~Liu, L.~Liu, et~al., Dapo: An open-source llm reinforcement learning system at scale, arXiv preprint arXiv:2503.14476 (2025).

\bibitem{magister2023teaching}
L.~C. Magister, J.~Mallinson, J.~Adamek, E.~Malmi, A.~Severyn, Teaching small language models to reason, in: Proceedings of the 61st Annual Meeting of the Association for Computational Linguistics (Volume 2: Short Papers), 2023, pp. 1773--1781.

\bibitem{dong2025agentic}
G.~Dong, H.~Mao, K.~Ma, L.~Bao, Y.~Chen, Z.~Wang, Z.~Chen, J.~Du, H.~Wang, F.~Zhang, et~al., Agentic reinforced policy optimization, arXiv preprint arXiv:2507.19849 (2025).

\bibitem{li2025torl}
X.~Li, H.~Zou, P.~Liu, Torl: Scaling tool-integrated rl, arXiv preprint arXiv:2503.23383 (2025).

\bibitem{lin2025understanding}
H.~Lin, Z.~Xu, Understanding tool-integrated reasoning, arXiv preprint arXiv:2508.19201 (2025).

\bibitem{peng2023rwkv}
B.~Peng, E.~Alcaide, Q.~Anthony, A.~Albalak, S.~Arcadinho, S.~Biderman, H.~Cao, X.~Cheng, M.~Chung, L.~Derczynski, et~al., Rwkv: Reinventing rnns for the transformer era, in: Findings of the association for computational linguistics: EMNLP 2023, 2023, pp. 14048--14077.

\bibitem{child2019generating}
R.~Child, S.~Gray, A.~Radford, I.~Sutskever, Generating long sequences with sparse transformers, arXiv preprint arXiv:1904.10509 (2019).

\bibitem{katharopoulos2020transformers}
A.~Katharopoulos, A.~Vyas, N.~Pappas, F.~Fleuret, Transformers are rnns: Fast autoregressive transformers with linear attention, in: International conference on machine learning, PMLR, 2020, pp. 5156--5165.

\bibitem{beltagy2020longformer}
I.~Beltagy, M.~E. Peters, A.~Cohan, Longformer: The long-document transformer, arXiv preprint arXiv:2004.05150 (2020).

\bibitem{zhao2019explicit}
G.~Zhao, J.~Lin, Z.~Zhang, X.~Ren, Q.~Su, X.~Sun, Explicit sparse transformer: Concentrated attention through explicit selection, arXiv preprint arXiv:1912.11637 (2019).

\bibitem{gu2021efficiently}
A.~Gu, K.~Goel, C.~R{\'e}, Efficiently modeling long sequences with structured state spaces, arXiv preprint arXiv:2111.00396 (2021).

\bibitem{ding2024longrope}
Y.~Ding, L.~L. Zhang, C.~Zhang, Y.~Xu, N.~Shang, J.~Xu, F.~Yang, M.~Yang, Longrope: Extending llm context window beyond 2 million tokens, arXiv preprint arXiv:2402.13753 (2024).

\bibitem{hu2025memory}
Y.~Hu, S.~Liu, Y.~Yue, G.~Zhang, B.~Liu, F.~Zhu, J.~Lin, H.~Guo, S.~Dou, Z.~Xi, et~al., Memory in the age of ai agents, arXiv preprint arXiv:2512.13564 (2025).

\bibitem{zhong2023memorybank}
W.~Zhong, L.~Guo, Q.~Gao, H.~Ye, Y.~Wang, Memorybank: Enhancing large language models with long-term memory, arXiv preprint arXiv:2305.10250 (2023).

\bibitem{chen2024meminsight}
Y.~Chen, X.~Li, Y.~Liu, W.~Zhang, T.-Y. Liu, Meminsight: Autonomous memory augmentation for llm agents, arXiv preprint arXiv:2407.09876 (2024).

\bibitem{fang2025memp}
R.~Fang, Y.~Liang, X.~Wang, J.~Wu, S.~Qiao, P.~Xie, F.~Huang, H.~Chen, N.~Zhang, Memp: Exploring agent procedural memory, arXiv preprint arXiv:2508.06433 (2025).

\bibitem{zhang2025memory}
Y.~Zhang, J.~Shu, Y.~Ma, X.~Lin, S.~Wu, J.~Sang, Memory as action: Autonomous context curation for long-horizon agentic tasks, arXiv preprint arXiv:2510.12635 (2025).

\bibitem{wang2026infmem}
X.~Wang, M.~Li, P.~Lu, X.-W. Chang, L.~Shang, J.~Li, F.~Mi, P.~Parthasarathi, Y.~Cui, Infmem: Learning system-2 memory control for long-context agent, arXiv preprint arXiv:2602.02704 (2026).

\bibitem{li2024chain}
Y.~Zhang, R.~Sun, Y.~Chen, T.~Pfister, R.~Zhang, S.~{\"O}. Arik, Chain of agents: Large language models collaborating on long-context tasks, Advances in Neural Information Processing Systems 37 (2024) 132208--132237.

\bibitem{wang2025mem}
Y.~Wang, R.~Takanobu, Z.~Liang, Y.~Mao, Y.~Hu, J.~McAuley, X.~Wu, Mem-$\{$$\backslash$alpha$\}$: Learning memory construction via reinforcement learning, arXiv preprint arXiv:2509.25911 (2025).

\bibitem{yuan2025memsearcher}
Q.~Yuan, J.~Lou, Z.~Li, J.~Chen, Y.~Lu, H.~Lin, L.~Sun, D.~Zhang, X.~Han, Memsearcher: Training llms to reason, search and manage memory via end-to-end reinforcement learning, arXiv preprint arXiv:2511.02805 (2025).

\bibitem{pignatelli2023survey}
E.~Pignatelli, J.~Ferret, M.~Geist, T.~Mesnard, H.~van Hasselt, O.~Pietquin, L.~Toni, A survey of temporal credit assignment in deep reinforcement learning, arXiv preprint arXiv:2312.01072 (2023).

\bibitem{sutton1988learning}
R.~S. Sutton, Learning to predict by the methods of temporal differences, Machine learning 3~(1) (1988) 9--44.

\bibitem{sutton2016emphatic}
R.~S. Sutton, A.~R. Mahmood, M.~White, An emphatic approach to the problem of off-policy temporal-difference learning, Journal of Machine Learning Research 17~(73) (2016) 1--29.

\bibitem{sutton2011horde}
R.~S. Sutton, J.~Modayil, M.~Delp, T.~Degris, P.~M. Pilarski, A.~White, D.~Precup, Horde: A scalable real-time architecture for learning knowledge from unsupervised sensorimotor interaction, in: The 10th international conference on autonomous agents and multiagent systems-volume 2, 2011, pp. 761--768.

\bibitem{andrychowicz2017hindsight}
M.~Andrychowicz, F.~Wolski, A.~Ray, J.~Schneider, R.~Fong, P.~Welinder, B.~McGrew, J.~Tobin, O.~Pieter~Abbeel, W.~Zaremba, Hindsight experience replay, Advances in neural information processing systems 30 (2017).

\bibitem{schmidhuber2019reinforcement}
J.~Schmidhuber, Reinforcement learning upside down: Don't predict rewards--just map them to actions, arXiv preprint arXiv:1912.02875 (2019).

\bibitem{janner2021offline}
M.~Janner, Q.~Li, S.~Levine, Offline reinforcement learning as one big sequence modeling problem, Advances in neural information processing systems 34 (2021) 1273--1286.

\bibitem{chen2021decision}
L.~Chen, K.~Lu, A.~Rajeswaran, K.~Lee, A.~Grover, M.~Laskin, P.~Abbeel, A.~Srinivas, I.~Mordatch, Decision transformer: Reinforcement learning via sequence modeling, Advances in neural information processing systems 34 (2021) 15084--15097.

\bibitem{li2024llms}
H.~Li, Q.~Dong, J.~Chen, H.~Su, Y.~Zhou, Q.~Ai, Z.~Ye, Y.~Liu, Llms-as-judges: a comprehensive survey on llm-based evaluation methods, arXiv preprint arXiv:2412.05579 (2024).

\bibitem{bo2024reflective}
X.~Bo, Z.~Zhang, Q.~Dai, X.~Feng, L.~Wang, R.~Li, X.~Chen, J.-R. Wen, Reflective multi-agent collaboration based on large language models, Advances in Neural Information Processing Systems 37 (2024) 138595--138631.

\bibitem{yang2018hotpotqa}
Z.~Yang, P.~Qi, S.~Zhang, Y.~Bengio, W.~Cohen, R.~Salakhutdinov, C.~D. Manning, Hotpotqa: A dataset for diverse, explainable multi-hop question answering, in: Proceedings of the 2018 conference on empirical methods in natural language processing, 2018, pp. 2369--2380.

\bibitem{hsieh2024ruler}
C.-P. Hsieh, S.~Sun, S.~Kriman, S.~Acharya, D.~Rekesh, F.~Jia, Y.~Zhang, B.~Ginsburg, Ruler: What's the real context size of your long-context language models?, arXiv preprint arXiv:2404.06654 (2024).

\bibitem{zelikman2024star}
E.~Zelikman, Y.~Wu, J.~Mu, N.~D. Goodman, Star: Self-taught reasoner bootstrapping reasoning with reasoning, in: Proc. the 36th International Conference on Neural Information Processing Systems, Vol. 1126, 2024.

\bibitem{muralidharan2024compact}
S.~Muralidharan, S.~Turuvekere~Sreenivas, R.~Joshi, M.~Chochowski, M.~Patwary, M.~Shoeybi, B.~Catanzaro, J.~Kautz, P.~Molchanov, Compact language models via pruning and knowledge distillation, Advances in Neural Information Processing Systems 37 (2024) 41076--41102.

\bibitem{yang2025qwen3}
A.~Yang, A.~Li, B.~Yang, B.~Zhang, B.~Hui, B.~Zheng, B.~Yu, C.~Gao, C.~Huang, C.~Lv, et~al., Qwen3 technical report, arXiv preprint arXiv:2505.09388 (2025).

\bibitem{sheng2024hybridflow}
G.~Sheng, C.~Zhang, Z.~Ye, X.~Wu, W.~Zhang, R.~Zhang, Y.~Peng, H.~Lin, C.~Wu, Hybridflow: A flexible and efficient rlhf framework, arXiv preprint arXiv: 2409.19256 (2024).

\bibitem{kamradt2023needle}
G.~Kamradt, \href{https://github.com/gkamradt/LLMTest\_NeedleInAHaystack}{{Needle In A Haystack - pressure testing LLMs}}.
\newline\urlprefix\url{https://github.com/gkamradt/LLMTest\_NeedleInAHaystack}

\bibitem{kovcisky2018narrativeqa}
T.~Ko{\v{c}}isk{\`y}, J.~Schwarz, P.~Blunsom, C.~Dyer, K.~M. Hermann, G.~Melis, E.~Grefenstette, The narrativeqa reading comprehension challenge, Transactions of the Association for Computational Linguistics 6 (2018) 317--328.

\bibitem{dasigi2021dataset}
P.~Dasigi, K.~Lo, I.~Beltagy, A.~Cohan, N.~A. Smith, M.~Gardner, A dataset of information-seeking questions and answers anchored in research papers, in: Proceedings of the 2021 Conference of the North American Chapter of the Association for Computational Linguistics: Human Language Technologies, 2021, pp. 4599--4610.

\bibitem{ho2020constructing}
X.~Ho, A.-K.~D. Nguyen, S.~Sugawara, A.~Aizawa, Constructing a multi-hop qa dataset for comprehensive evaluation of reasoning steps, in: Proceedings of the 28th International Conference on Computational Linguistics, 2020, pp. 6609--6625.

\bibitem{trivedi2022musique}
H.~Trivedi, N.~Balasubramanian, T.~Khot, A.~Sabharwal, Musique: Multihop questions via single-hop question composition, Transactions of the Association for Computational Linguistics 10 (2022) 539--554.

\bibitem{devlin2019bert}
J.~Devlin, M.-W. Chang, K.~Lee, K.~Toutanova, Bert: Pre-training of deep bidirectional transformers for language understanding, in: Proceedings of the 2019 conference of the North American chapter of the association for computational linguistics: human language technologies, volume 1 (long and short papers), 2019, pp. 4171--4186.

\bibitem{kwon2023efficient}
W.~Kwon, Z.~Li, S.~Zhuang, Y.~Sheng, L.~Zheng, C.~H. Yu, J.~Gonzalez, H.~Zhang, I.~Stoica, Efficient memory management for large language model serving with pagedattention, in: Proceedings of the 29th symposium on operating systems principles, 2023, pp. 611--626.

\end{thebibliography}

\clearpage
\appendix
\renewcommand{\thetable}{A\arabic{table}}
\renewcommand{\thefigure}{A\arabic{figure}}
\renewcommand{\thetheorem}{A\arabic{theorem}}
\setcounter{table}{0}
\setcounter{figure}{0}
\setcounter{theorem}{0}

\section{Theoretical Proof of Theorem 1}
\label{appendix:proof_of_theorem1}
\begin{theorem}[Information-Theoretic Decomposition of the TAMTRL Objective]
Consider an optimization step at state $S_t = (q, D_t, M_t)$ within a multi-step reasoning process. Let $\pi_\theta(\cdot \mid S_t)$ denote the policy generating the next memory $M_{t+1}$, and let $r_i \in \{0,1\}$ be the binary indicator of final task success, whose distribution depends on $M_{t+1}$ and the subsequent rollout. Given a teacher log-likelihood score $\hat{p}_{\text{t}} = \log \pi_{\text{teacher}}(M_{t+1} \mid S_t)$ and a reference policy $\pi_{\text{ref}}$, the TAMTRL objective is defined as:
\begin{equation}
    \mathcal{J}(\theta) \;=\; \mathbb{E}_{M_{t+1} \sim \pi_\theta}\bigl[\hat{p}_{\text{t}} \cdot r_i \bigr] \;-\; \beta \, D_{\mathrm{KL}}\bigl[\pi_\theta(\cdot\mid S_t) \;\|\; \pi_{\text{ref}}(\cdot\mid S_t)\bigr]. \nonumber
\end{equation}
This objective $\mathcal{J}(\theta)$ can be exactly decomposed into a weighted sum comprising a success-conditional optimization term, a failure-conditional regularization term, and a memory-reward mutual information term:
\begin{equation}
    \mathcal{J}(\theta) \;=\; P(r_i=1 \mid S_t) \cdot \mathcal{L}_{\text{succ}}(\theta) \;+\; P(r_i=0 \mid S_t) \cdot \mathcal{L}_{\text{fail}}(\theta) \;+\; \beta \, I_{\pi_\theta}(M_{t+1} ; r_i \mid S_t), \nonumber
\end{equation}
where the components are defined as:
\begin{align}
    \mathcal{L}_{\text{succ}}(\theta) &= \mathbb{E}_{\pi_\theta}\bigl[\log \pi_{\text{teacher}} \mid r_i=1\bigr] - \beta \, D_{\mathrm{KL}}\bigl[\pi_\theta(\cdot \mid S_t, r_i=1) \;\|\; \pi_{\text{ref}}(\cdot \mid S_t)\bigr], \nonumber \\
    \mathcal{L}_{\text{fail}}(\theta) &= - \beta \, D_{\mathrm{KL}}\bigl[\pi_\theta(\cdot \mid S_t, r_i=0) \;\|\; \pi_{\text{ref}}(\cdot \mid S_t)\bigr], \nonumber
\end{align}
and $I_{\pi_\theta}(M_{t+1} ; r_i \mid S_t)$ represents the mutual information between the generated memory $M_{t+1}$ and the outcome $r_i$ given state $S_t$.
\end{theorem}

\begin{proof}
We begin by decomposing the Kullback-Leibler (KL) divergence term. Using the definition of KL divergence and the expansion of entropy, we have:
\begin{equation}
    D_{\mathrm{KL}}\bigl[\pi_\theta(\cdot|S_t) \| \pi_{\text{ref}}(\cdot|S_t)\bigr] = - H_{\pi_\theta}(M_{t+1} \mid S_t) - \mathbb{E}_{M_{t+1} \sim \pi_\theta}\bigl[\log \pi_{\text{ref}}(M_{t+1} \mid S_t)\bigr]. \label{eq:kl_def}
\end{equation}
Invoking the information-theoretic identity $H(X) = H(X|Y) + I(X;Y)$, we decompose the conditional entropy of $M_{t+1}$ with respect to the binary outcome $r_i$:
\begin{equation}
    H_{\pi_\theta}(M_{t+1} \mid S_t) = \underbrace{\sum_{r \in \{0,1\}} P(r_i=r \mid S_t) H_{\pi_\theta}(M_{t+1} \mid S_t, r_i=r)}_{H_{\pi_\theta}(M_{t+1} \mid S_t, r_i)} \;+\; I_{\pi_\theta}(M_{t+1} ; r_i \mid S_t).
\end{equation}
Similarly, we expand the expectation of the log-reference policy over the values of $r_i$:
\begin{equation}
    \mathbb{E}_{\pi_\theta}\bigl[\log \pi_{\text{ref}}\bigr] = \sum_{r \in \{0,1\}} P(r_i=r \mid S_t) \, \mathbb{E}_{\pi_\theta}\bigl[\log \pi_{\text{ref}} \mid r_i=r\bigr].
\end{equation}
Substituting these expansions back into Eq.~\eqref{eq:kl_def} and regrouping terms by $r_i$, we obtain:
\begin{equation}
    \begin{aligned}
    D_{\mathrm{KL}}\bigl[\pi_\theta \| \pi_{\text{ref}}\bigr] = & \sum_{r \in \{0,1\}} P(r_i=r \mid S_t) \, D_{\mathrm{KL}}\bigl[\pi_\theta(\cdot \mid S_t, r_i=r) \;\|\; \pi_{\text{ref}}(\cdot \mid S_t)\bigr] \\
    & - I_{\pi_\theta}(M_{t+1} ; r_i \mid S_t).
    \end{aligned}
\end{equation}
Next, we examine the reward term. Since $r_i$ is a binary indicator, the expectation is non-zero only when $r_i=1$:
\begin{equation}
    \mathbb{E}_{\pi_\theta}\bigl[\hat{p}_{\text{t}} \cdot r_i\bigr] = P(r_i=1 \mid S_t) \cdot \mathbb{E}_{\pi_\theta}\bigl[\log \pi_{\text{teacher}} \mid r_i=1\bigr].
\end{equation}
Finally, substituting the decomposed KL term (multiplied by $-\beta$) and the reward term back into the original objective $\mathcal{J}(\theta)$, we get:
\begin{align*}
    \mathcal{J}(\theta) &= P(r_i=1 \mid S_t) \Bigl( \mathbb{E}_{\pi_\theta}[\log \pi_{\text{teacher}} \mid r_i=1] - \beta D_{\mathrm{KL}}[\pi_\theta(\cdot \mid r_i=1) \| \pi_{\text{ref}}] \Bigr) \\
    &\quad + P(r_i=0 \mid S_t) \Bigl( - \beta D_{\mathrm{KL}}[\pi_\theta(\cdot \mid r_i=0) \| \pi_{\text{ref}}] \Bigr) \\
    &\quad + \beta \, I_{\pi_\theta}(M_{t+1} ; r_i \mid S_t).
\end{align*}
Identifying the terms in parentheses as $\mathcal{L}_{\text{succ}}(\theta)$ and $\mathcal{L}_{\text{fail}}(\theta)$ respectively concludes the proof.
\end{proof}

\section{Experimental Details}
\subsection{Datasets for TAMTRL Main Experiments}
\label{appendix:datasets_statics}
We summarize the number of samples in the training and evaluation datasets in Table~\ref{tab:datasets_statics}. For all RL training, we use the synthetic HotpotQA-train dataset. For the SFT, STaR, and Vanilla-KD methods, we use the same data as in RL, processed to obtain the training data. Each trajectory is then split into multiple sample pairs based on the turn for training.

\begin{table}[t]
\centering
\small
\setlength{\tabcolsep}{4pt}
\renewcommand{\arraystretch}{1.1}
\begin{tabular}{ll}
\toprule
\textbf{Dataset} & \textbf{\# Train / Test} \\
\midrule
HotpotQA-train & 25600 / - \\
SFT-train & 65057 / - \\
STaR-train-0.6B & 5264 / - \\
STaR-train-1.7B & 60773 / - \\
Vanilla-KD-train & 80040 / - \\
HotpotQA-test & - / 128 \\
RULER-QA & - / 500 \\
NIAH & - / 500 \\
2Wikimqa & - / 200 \\
Musique & - / 200 \\
Narrativeqa & - / 200 \\
Qasper & - / 200 \\
\bottomrule
\end{tabular}
\caption{Statistics of the training and test datasets used by all methods.}
\label{tab:datasets_statics}
\end{table}

\subsection{Implementation Details}
\label{appendix:implementation_details}
Following MemAgent~\cite{yu2025memagent}, we use the DAPO~\cite{yu2025dapo} algorithm for all RL-based method. For the LLM-judge method, we use the non-thinking mode of Qwen3-8B as the judge to perform turn-level credit assignment, with the template shown in Figure~\ref{fig:judge_prompt}. For PRM, we train a BERT-based~\cite{devlin2019bert} process reward model on the STaR-train-1.7B dataset. Specifically, for each turn, we use the corresponding LLM input–output pair as the input to the PRM, and employ the final outcome reward as the supervision signal to predict correctness. The predicted probability of being correct is then used as the turn-level process reward for intermediate supervision. We train the model using a binary cross-entropy loss with a learning rate of 2e-5 for one epoch, achieving a final classification accuracy of 96.05\%. Following the analysis in Section~\ref{sebsec:reward_design_explore}, and similar to TAMTRL, we multiply all turn-level scores with the outcome reward to obtain the final reward for each turn, which helps avoid gradient conflicts and stabilize optimization. For all RL methods, we set a KL factor of \(1 \times 10^{-3}\) and disable the entropy loss. The AdamW optimizer with \(\beta_1 = 0.9\), \(\beta_2 = 0.95\), weight decay of 0.01, and a constant learning rate of \(1 \times 10^{-6}\) is employed, along with a linear warm-up scheduler with a warm-up step of 20. We use a rollout batch size of 32 and a group size of 8, accelerated via vLLM~\cite{kwon2023efficient}. For all distillation methods, we use a learning rate of \(1 \times 10^{-5}\) and train for 5 epochs. For SFT, we use Qwen3-14B to generate CoT reasoning traces on the training dataset, filtering the correct traces for fine-tuning. For Vanilla-KD, Qwen3-8B is used as the teacher model. For STaR, we roll out 8 responses per question and filter the correct responses for CoT distillation. For all RL methods, we use a learning rate of \(1 \times 10^{-6}\) and train for 800 steps with the same parameter settings as MemAgent. All experiments were conducted on 8 × A100 GPUs, each with 80 GB of memory. The statistics of the training data used by all methods are shown in Table~\ref{tab:datasets_statics}.

\subsection{Comparison of Computational Cost}
\label{appendix:computational_cost}
To compare the training overhead of different methods, we evaluate their training cost using Qwen3-0.6B as the backbone model on an 8$\times$A100 GPU setup with 80 GB total memory, as summarized in Table~\ref{tab:time_cost}. SFT and Vanilla-KD require external models to generate training data, resulting in relatively long data processing times. In contrast, STaR relies solely on self-generated signals, leading to faster processing; however, due to the model’s limited capability, the generated data are of lower quality, which ultimately yields inferior performance. RL-based methods achieve a better balance, offering stronger performance with moderate time overhead. LLM-judge requires reward signals provided by LLM rollouts for training, resulting in relatively slow training. PRM-based approaches further exacerbate this issue, as they involve an additional data collection stage, followed by 2.05 hours of PRM training and 38.66 hours of RL optimization, leading to the highest overall training overhead. For TAMTRL, the computation of teacher probability scores introduces only a minor additional cost, and, owing to the shorter average response length of the model (see Fig.~\ref{fig:training_dynamics_06b}), the overall training time is lower than that of MemAgent, further demonstrating the efficiency advantage of TAMTRL.

\begin{table}[t]
\centering
\small
\setlength{\tabcolsep}{4pt}
\renewcommand{\arraystretch}{1.1}
\begin{tabular}{lccc}
\toprule
\textbf{Method} & \textbf{Data Processing} & \textbf{Training} & \textbf{Total} \\
\midrule
SFT & 54.12 & 3.7 & 58.82 \\
STaR & 23.33 & 0.32 & 23.65 \\
Vanilla-KD & 37.66 & 21.59 & 59.25 \\
MemAgent & 0 & 37.75 & 37.75 \\
LLM-judge & 0 & 48.41 & 48.41 \\
PRM & 41.06 & 40.71 & 81.77 \\
TAMTRL & 0 & 33.79 & 33.79 \\
\bottomrule
\end{tabular}
\caption{Estimated computational time (hours) for various methods with Qwen3-0.6B.}
\label{tab:time_cost}
\end{table}

\section{Pseudocode of TAMTRL}
\label{appendix:pseudocode}
We present the pseudocode of TAMTRL in Algorithm~\ref{alg:tamtrl}.

\begin{algorithm*}[t]
\caption{TAMTRL: Teacher-Aligned Reward Reshaping for Multi-Turn Reinforcement Learning}
\begin{algorithmic}[1]
\REQUIRE Initial policy model $\pi_\theta$; task prompts $\mathcal{D}$; clean task prompts $\mathcal{C}$; hyperparameters $\varepsilon_{\text{low}}, \varepsilon_{\text{high}}$
\ENSURE Optimized policy model $\pi_\theta$

\FOR{step $= 1, \ldots, S$}
    \STATE Sample a batch $\mathcal{D}_b$ from $\mathcal{D}$
    \STATE Retrieve the corresponding clean prompts $\mathcal{C}_b$
    \STATE Update the old policy model $\pi_{\theta_{\text{old}}} \leftarrow \pi_\theta$
    \STATE Sample $G$ trajectories $\{o^i_j\}_{j=1}^G \sim \pi_{\theta_{\text{old}}}(\cdot \mid q^i)$ for each $q^i \in \mathcal{D}_b$
    
    \STATE Compute normalized teacher probability scores $\hat{p}^i_{tj}$ using $\theta_{\text{old}}$
    
    \STATE Compute outcome rewards $\{r^i_j\}_{j=1}^G$ via Exact Match (EM)
    \STATE Compute final rewards $R^i_{tj} = \hat{p}^i_{tj} \cdot r^i_j$
    
    \STATE Filter out low-quality samples and add the remaining trajectories to the dynamic sampling buffer
    \IF{buffer size $n_b < N$}
        \STATE \textbf{continue}
    \ENDIF
    
    \STATE For each turn-level trajectory $o^i_{tj}$ in the buffer, compute token-level advantages $\hat{A}^i_{tj}$ using Equation (7)
    
    \FOR{iteration $= 1, \ldots, \mu$}
        \STATE Update the policy model $\pi_\theta$ by maximizing the DAPO objective using Equation (3)
    \ENDFOR
\ENDFOR

\RETURN $\pi_\theta$
\end{algorithmic}
\label{alg:tamtrl}
\end{algorithm*}

\section{Prompts}
\label{appendix:prompts}
We used the same prompts as MemAgent, as shown in Figure~\ref{fig:memagent}. Follow~\cite{kim2024prometheus}, we design the prompt for the LLM-judge method as shown in Figure~\ref{fig:judge_prompt} to perform turn-level credit assignment.

\begin{figure*}[htbp]
    \centering
    \includegraphics[width=1\textwidth]{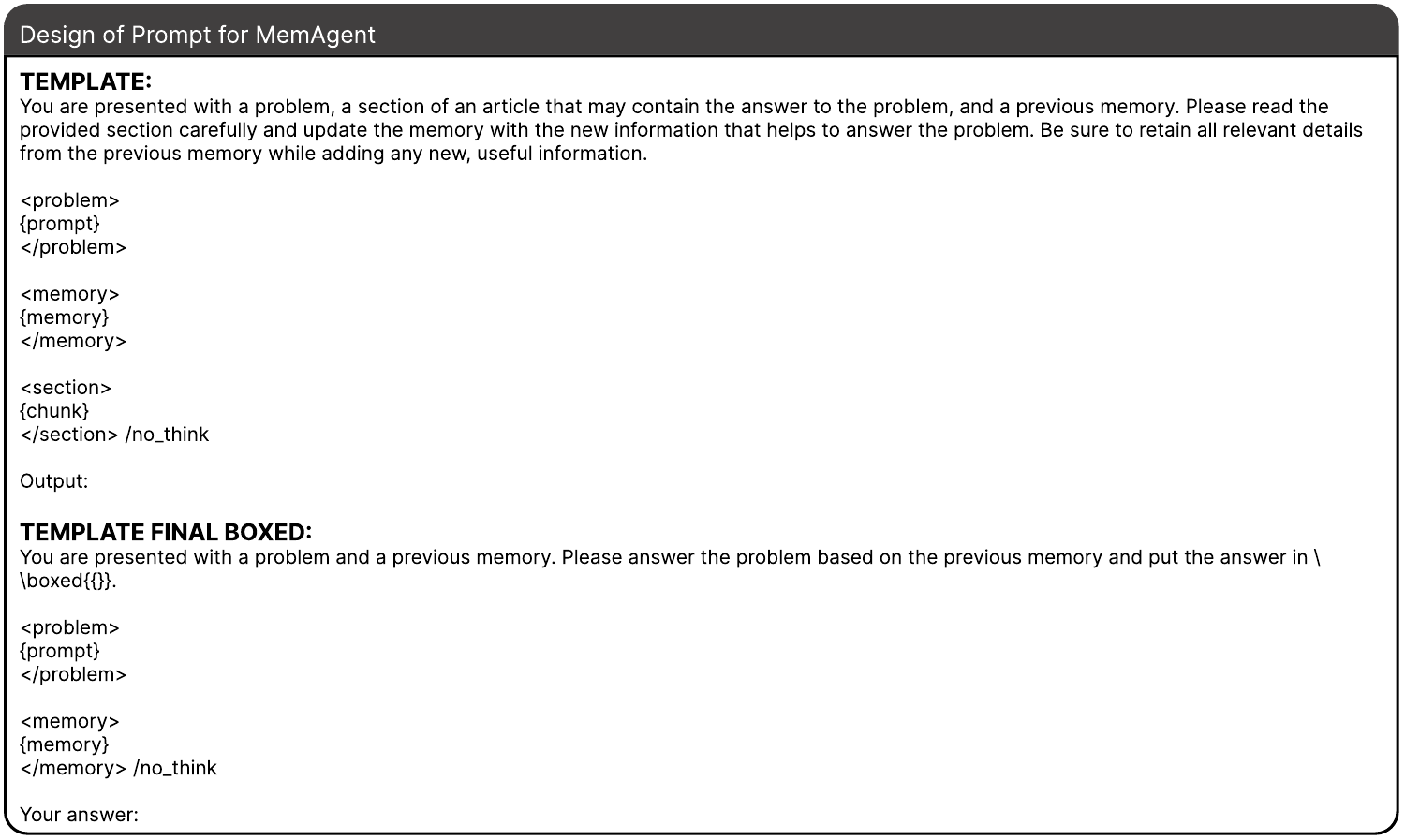}
    \caption{Design of Prompt for MemAgent.}
    \label{fig:memagent}
\end{figure*}

\begin{figure*}[htbp]
    \centering
    \includegraphics[width=1\textwidth]{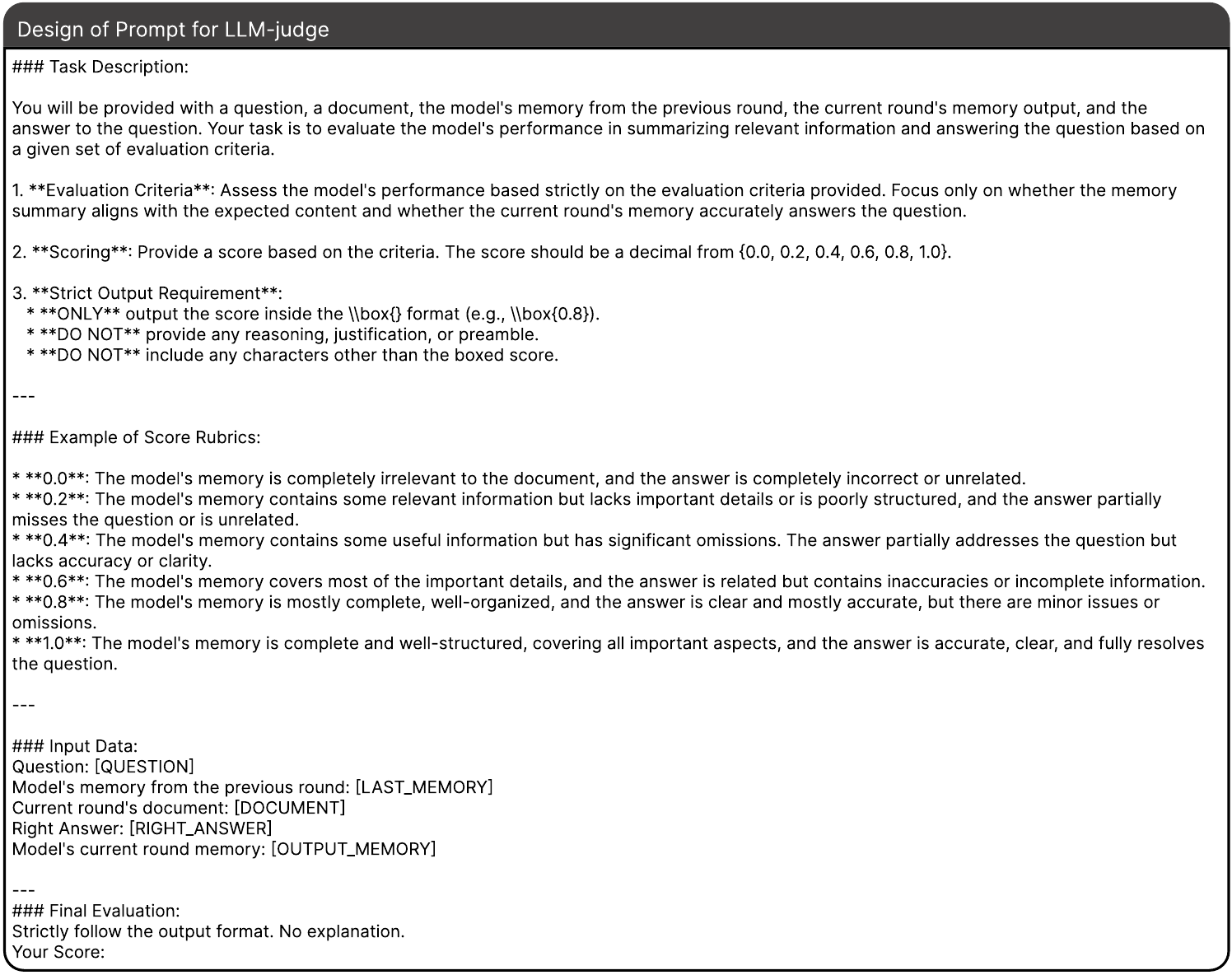}
    \caption{Design of Prompt for LLM-judge.}
    \label{fig:judge_prompt}
\end{figure*}

\section{Case Study}
\label{appendix:case_study}
We present output examples of the TAMTRL-0.6B and TAMTRL-1.7B models on the HotpotQA dataset in Figures~\ref{fig:06b_case1} and~\ref{fig:17b_case1}. The models effectively identify salient evidence through complex multi-turn interactions and produce correct answers, demonstrating robust long-context reasoning capability.

\begin{figure*}[htbp]
    \centering
    \includegraphics[width=1\textwidth]{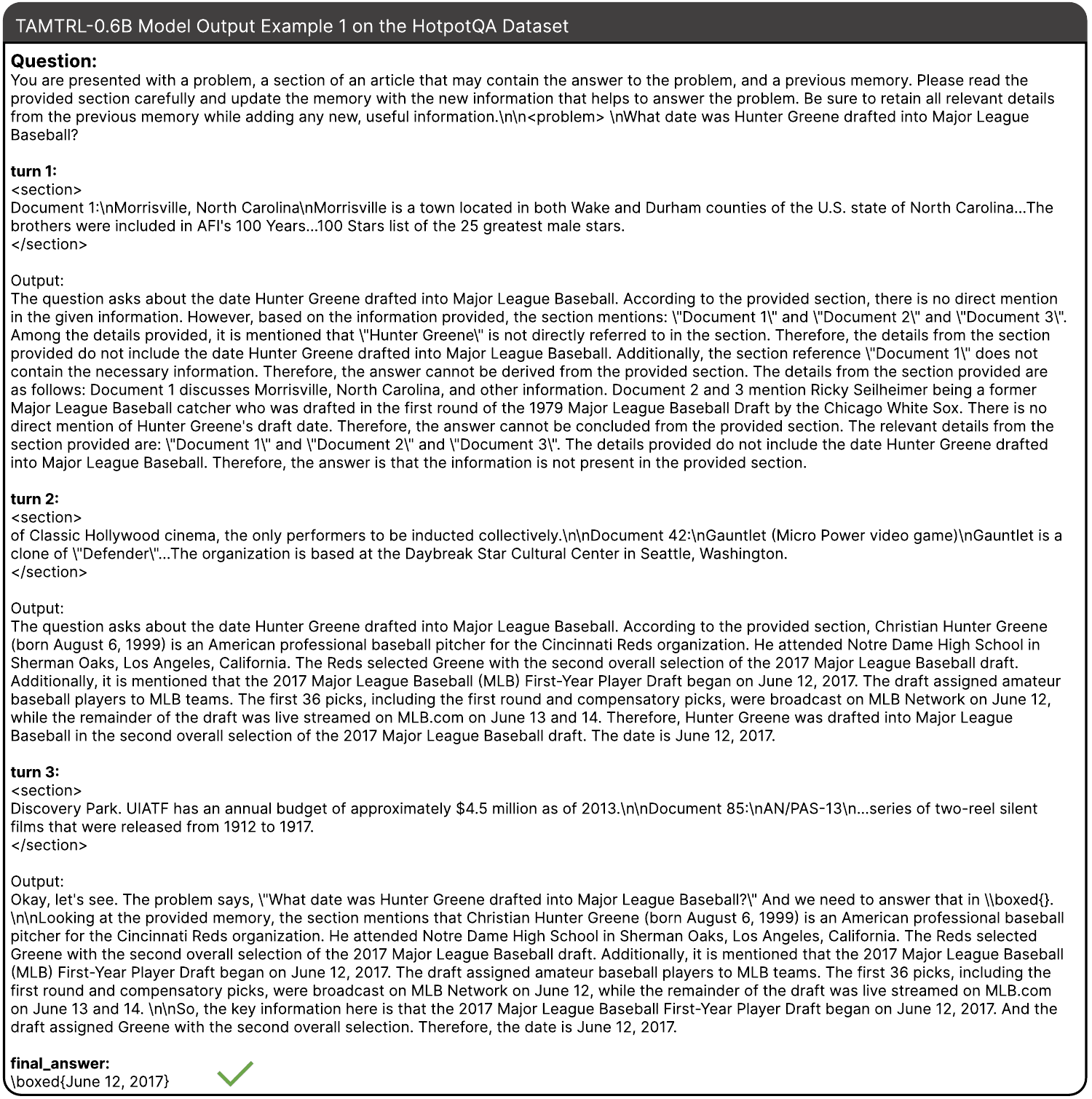}
    \caption{TAMTRL-0.6B model output example 1 on the hotpotQA dataset.}
    \label{fig:06b_case1}
\end{figure*}

\begin{figure*}[htbp]
    \centering
    \includegraphics[width=1\textwidth]{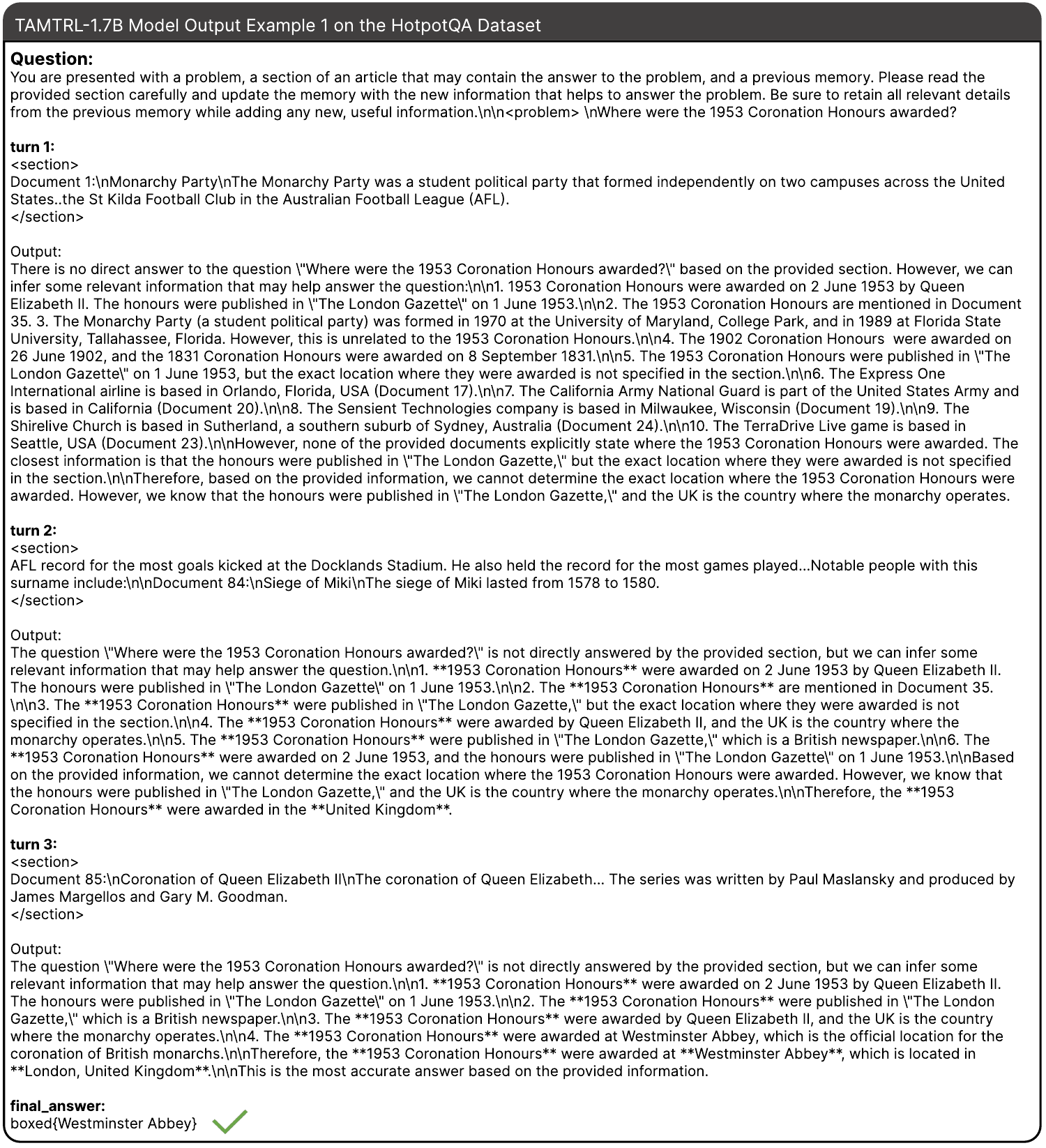}
    \caption{TAMTRL-1.7B model output example 1 on the hotpotQA dataset.}
    \label{fig:17b_case1}
\end{figure*}

\end{document}